\documentclass[twoside,11pt]{article} 

\usepackage[utf8]{inputenc}
\usepackage[T1]{fontenc}
\usepackage[colorlinks=false,allbordercolors={1 1 1}]{hyperref}
\usepackage{url}
\usepackage{booktabs}
\usepackage{nicefrac}
\usepackage{microtype}
\usepackage{dsfont}
\usepackage{amsmath, amssymb, amsfonts}
\usepackage{mathtools}
\usepackage{thmtools}
\usepackage[page]{appendix} %
\usepackage{natbib} %
\usepackage{xcolor}
\usepackage{booktabs}
\usepackage{multirow}
\usepackage{graphicx}
\usepackage{enumerate}
\usepackage[capitalise,noabbrev]{cleveref} %
\usepackage[nohyperref]{jmlr2e} %

\crefname{appsec}{Appendix}{Appendices}
\crefformat{equation}{(#2#1#3)}

\newtheorem{assumption}{Assumption}

\newcommand{\E}{\mathop{\mathbb{E}}}
\newcommand{\V}{\mathop{\mathbb{V}}}
\newcommand{\C}{\mathop{\mathbb{C}}}
\newcommand{\R}{\mathbb{R}}
\newcommand{\tran}{^{\top}}
\newcommand{\patch}{\scalebox{0.5}{$\bigstar$}}
\DeclareMathOperator{\diag}{diag}
\DeclareMathOperator{\trace}{Tr}

\DeclareMathOperator{\vect}{vec}
\DeclareMathOperator{\sigmoid}{sigmoid}
\DeclareMathOperator{\swish}{swish}
\DeclareMathOperator{\ReLU}{ReLU}
\DeclareMathOperator{\I}{I}

\DeclareMathOperator{\cov}{cov}
\DeclareMathOperator{\var}{var}
\newcommand{\1}{\mathds{1}}
\newcommand{\Dt}{\Delta t}
\newcommand{\xdt}{\boldsymbol{x}}
\newcommand{\gdt}{\boldsymbol{g}}

\let\norm\undefined
\DeclarePairedDelimiter\norm{\lVert}{\rVert}
\DeclarePairedDelimiterX{\dprod}[2]{\langle}{\rangle}{#1, #2}
\allowdisplaybreaks

\newcommand{\mux}{\mu}
\newcommand{\sigmax}{\sigma}

\jmlrheading{}{}{}{}{}{}{Stefano Peluchetti and Stefano Favaro}

\ShortHeadings{Doubly infinite residual neural networks: a diffusion process approach}{Peluchetti and Favaro}
\firstpageno{1}

\usepackage{lastpage}
\jmlrheading{22}{2021}{1-\pageref{LastPage}}{7/20}{8/21}{20-706}{Stefano Peluchetti and Stefano Favaro}
\ShortHeadings{Doubly infinite residual neural networks: a diffusion process approach}{Peluchetti and Favaro}

\begin{document}

\title{Doubly infinite residual neural networks:\\ a diffusion process approach}
\author{\name Stefano Peluchetti \email speluchetti@cogent.co.jp\\
\addr Cogent Labs\\
106-0032, Tokyo, Japan.
\AND
\name Stefano Favaro \email stefano.favaro@unito.it\\ 
\addr Department of Economics and Statistics\\ 
University of Torino and Collegio Carlo Alberto\\
10122, Torino, Italy.}
\editor{Philipp Hennig}

\maketitle

\begin{abstract}
Modern neural networks featuring a large number of layers (depth) and units per layer (width) have achieved a remarkable performance across many domains. While there exists a vast literature on the interplay between infinitely wide neural networks and Gaussian processes, a little is known about analogous interplays with respect to infinitely deep neural networks. Neural networks with independent and identically distributed (i.i.d.) initializations exhibit undesirable forward and backward propagation properties as the number of layers increases, e.g., vanishing dependency on the input, and perfectly correlated outputs for any two inputs. To overcome these drawbacks, \cite{pelux2020sde} considered fully-connected residual networks (ResNets) with network's parameters initialized by means of distributions that shrink as the number of layers increases, thus establishing an interplay between infinitely deep ResNets and solutions to stochastic differential equations, i.e. diffusion processes, and showing that infinitely deep ResNets does not suffer from undesirable forward-propagation properties. In this paper, we review the results of \cite{pelux2020sde}, extending them to convolutional ResNets, and we establish analogous backward-propagation results, which directly relate to the problem of training fully-connected deep ResNets. Then, we investigate the more general setting of doubly infinite neural networks, where both network's width and network's depth grow unboundedly. We focus on doubly infinite fully-connected ResNets, for which we consider i.i.d. initializations. Under this setting, we show that the dynamics of quantities of interest converge, at initialization, to deterministic limits. This allow us to provide analytical expressions for inference, both in the case of weakly trained and fully trained ResNets. Our results highlight a limited expressive power of doubly infinite ResNets when the unscaled network's parameters are i.i.d. and the residual blocks are shallow.
\end{abstract}


\section{Introduction}

Modern neural networks featuring a large number of layers (depth) and units per layer (width) have achieved a remarkable performance across many domains \citep{lecun2015deep}. Under suitable distributional assumptions for network's parameters, the large-width limit of a neural network is a Gaussian process \citep{neal1995bayesian,matthews2018gaussian}. Such an interplay between infinitely wide neural networks and Gaussian processes has contributed to the study of properties of neural networks, and most recently it has become a critical tool for introducing inferential algorithms that directly target the infinite-dimensional setting \citep{lee2018deep,garriga-alonso2018deep,lee2019wide,arora2019exact}. According to recent studies on infinitely wide neural networks, it seems natural to ask whether there exists an interplay between infinitely deep neural networks and classes of stochastic processes. At a first glance, this interplay might prove elusive. Because of the duality between initialization schemes and Bayesian neural networks, it is well-known that any initialization scheme for neural networks may be interpreted as a prior distribution on network's parameters, and in turns the initialization induces a prior on the network. Therefore, any neural network at initialization may be viewed as a suitable stochastic process indexed by depth, whose distribution is defined through a sequence of conditional distributions mapping from each layer to the next layer. Early works have focused on the stabilization of the variance of key quantities of interest across the layers of neural networks \citep{glorot2010understanding,he2015delving}. More recent works \citep{poole2016exponential,schoenholz2017deep, hayou2019impact} considered the impact of initialization schemes to the forward-propagation of the input signal.

Even when initialized on the edge of chaos (EOC) for optimal forward-propagation of the signal \citep{hayou2019impact}, neural networks with an independent and identically distributed (i.i.d.) initialization exhibit pathological properties as their total depth increases. Arguably, the two most common pathological properties are: i) the dependency on the input signal eventually vanishes for most activation functions \citep{neal1995bayesian,poole2016exponential,schoenholz2017deep}; ii) the distributions of the layers, when viewed as random functions on the input space, eventually concentrate on restrictive families including constant functions \citep{hayou2019impact}. As an illustrative example, in \cref{fig:fspace_T_example} we show functions sampled from the last layer of a feedforward neural network for two activation functions under EOC initialization. For the hyperbolic tangent ($\tanh$) activation function, i.e. $\phi(x)=\tanh(x)$, the input signal has no discernible impact on the output, as can be seen by the constant marginal distributions, and the sampled functions are almost constant. This behavior is representative of most classes of smooth activation functions used in practice \citep{hayou2019impact}. For the rectified linear unit ($\ReLU$) activation function, i.e.  $\phi(x) = \max(0,x)$, the input signal affects the variance of the output and the sampled functions are piece-wise linear functions. In both cases, the outputs corresponding to any two input signals end up perfectly correlated. While this analysis applies to feedforward neural networks, residual neural networks (ResNets) suffer from analogous issues \citep{yang2017mean}. In addition, for the ResNets it is known that the variance of the Gaussian-distributed pre-activations may grow unbounded over the layers. 

\begin{figure}[htbp]
    \centering
    \hspace*{-0.5cm}\includegraphics[width=0.8\linewidth]{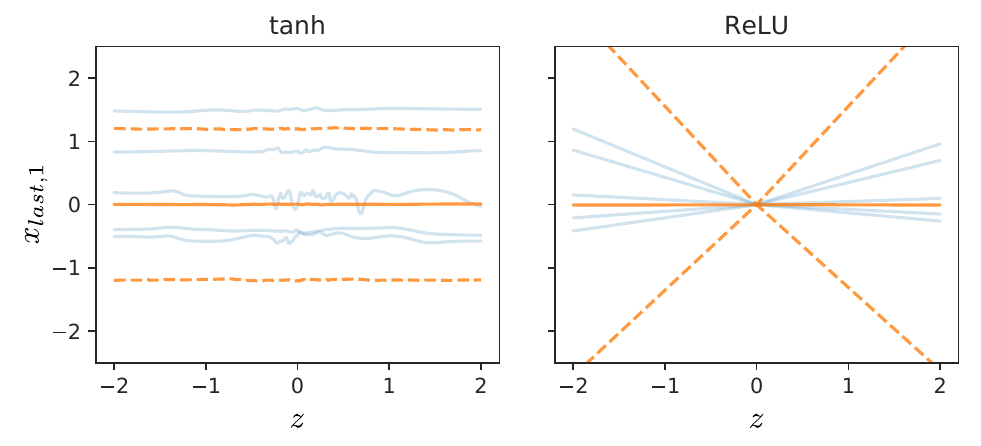}
    \caption{Samples of a given pre-activation from the last layer $x_{last,1}$ of a feedforward neural network with 500 layers of 500 units over an input signal $z \in [-2,2]$, and $\tanh$ and $\ReLU$ activation functions. Initialization on the EOC. Samples are displayed in blue, and for each input signal the $5\%$, $50\%$ and $95\%$ quantiles are displayed in orange.}
    \label{fig:fspace_T_example}
\end{figure}

The above-mentioned pathological properties, as well as other critical difficulties, are determined by the fact that common prior distributions on network's parameters introduce a constant level of randomness over each network's layer. To overcome these difficulties, \cite{pelux2020sde} introduced a prior distribution that depends on the number of layers, in such a way that the distribution of network's parameters shrinks as the number layers increases. Under this novel prior construction, \cite{pelux2020sde} showed that  fully-connected ResNets converge to certain classes of diffusion processes on a finite time interval, jointly over multiple input signals, as the number of layers increases. The conditions required for attaining the convergence to diffusion processes provide with a concrete guideline for selecting compatible neural network architectures, activation functions and distributions for network's parameters. The resulting limiting diffusion processes satisfy stochastic differential equations (SDE) that describe the evolution of infinitely deep neural networks over time (depth). The connection with SDEs sheds light on properties of very deep neural networks in a general framework, which includes finitely wide neural networks and correlated distributions for network's parameters. In particular,  \cite{pelux2020sde} showed that the limiting diffusion process is a well-behaved stochastic process in the sense that: i) it retains the dependency from the input signal; ii) it does not suffer from the perfect correlation constraint; iii) it does not collapse to a deterministic function nor it diverges. 

\subsection{Our contributions}

In this paper, we review the forward-propagation results of \cite{pelux2020sde}, and we extend them to convolutional ResNets. Then, we study analogous backward-propagation results, which relate to the fundamental problem of training ResNets. Stochastic gradient descent (SGD) is arguably the most common paradigm for training neural networks \citep{robbins1951stochastic,bottou2018optimization}. Focusing on the gradient backward-propagation, in neural networks we face a large number of Jacobian matrix multiplications for computing the gradients with respect to the networks' parameters of the lowest layers. This may result in a vanishing (or exploding) gradient, that is the gradient's magnitude in the the lowest layers goes to zero (or grows unbounded) as the number of layers increases. As SGD relies on gradients to perform the updates of the parameters, a vanishing (or exploding) gradient is detrimental to the training performance. The information propagation literature covers this setting, with results qualitatively similar to the forward-propagation analysis  \citep{schoenholz2017deep,hayou2019impact,yang2017mean}. Our study is critical to the performance of the training. We show that the Jacobian matrix of any layer with respect to the input layer converges to a matrix-valued diffusion process, which is the solution of a matrix SDE. Moreover, under appropriate non-explosivity conditions such a matrix-valued limiting diffusion process is shown to be invertible with dynamics given by a related matrix-valeud SDE. These results imply that in large-depth limit the Jacobian of the final layer with respect to any layer is again well-behaved and that exploding gradients are not possible.

We conclude our study by extending the results of \cite{pelux2020sde}, as well as the corresponding backward-propagation results presented in this paper, to the general setting of infinitely deep and infinitely wide ResNets. For short, these neural networks are referred to as doubly infinite ResNets. In such a setting, we assume that network's parameters are i.i.d., and we focus on doubly infinite ResNets defined through a restricted class of activation functions. With regards to the forward-propagation analysis of doubly infinite ResNets, we show that the dynamic of the neural network simplifies, and quantities of interest are either analytically available or can be efficiently approximated by numerical computations. Moreover, we show that the distribution in function space of a doubly infinite ResNet converges to the distribution of a Gaussian process with an affine kernel. With regards to the backward-propagation analysis, the recent literature on the Neural Tangent Kernel (NTK) investigates problem of training a neural network by means of gradient descent with infinitesimally small learning rate and quadratic loss for infinitely wide neural networks \citep{jacot2018neural,arora2019exact}. We establish a connection with this line of research, showing that for doubly infinite ResNets the NTK at initialization converges to an affine kernel. Under suitable assumptions for our class of doubly infinite ResNets, this result implies that weakly and fully trained neural networks with a large depth and width collapse to linear regression.

\subsection{Organization of the paper}

The paper is structured as follows. \cref{sec:preliminaries_diffusion} contains some preliminary results on diffusion process approximations of discrete-time stochastic processes, and the notation to be used throughout the paper. In \cref{sec:residual_diffusions} we review the forward-propagation results of \cite{pelux2020sde}, and we extend them to convolutional ResNets. \cref{sec:residual_diffusion_gradient} contains the backward-propagation analysis for infinitely deep ResNets, whereas \cref{sec:doubly_infinite} contains forward-propagation and backward-propagation analyses for doubly infinite ResNets. In \cref{sec:numerical_experiments} we present numerical experiments, and in \cref{sec:conclusion} we discuss our work and directions for future work. Proofs of our results and additional plots are deferred to \cref{app:proofs} and \cref{app:additional_plots}, respectively.

\section{Preliminaries on diffusion process approximations, and notation} \label{sec:preliminaries_diffusion}

We recall assumptions and results for diffusion process approximations of discrete-time stochastic processes, which are at the basis of the work of \cite{pelux2020sde}. Let $\xdt_{l}$, for $l=1,\dots,L$, denote the $l$-th layer of a neural network with $L$ layers, and let $\xdt_0$ be the network's input; we refer to the next section for a precise description of $\xdt_l$ in the context of neural networks, and in particular for ResNets. As we consider limiting continuous-time stochastic process, we re-index $\xdt_0,\xdt_1,\dots,\xdt_L$ on a discrete-time scale. In particular, let $T > 0$ be a terminal time, $\Dt = T/L$, for each $L$ we establish a correspondence between discrete indices $l \in \mathbb{Z}_{+}$ and discrete times $t \in \R_{+}$ by $l = 0,1,\dots,L \leftrightarrow t = 0,\,\Dt,\,2\Dt,\,\dots,\,T$.  Without loss of generality, we consider a neural network with input $\xdt_0$ and layers $\xdt_{\Dt},\dots,\xdt_T$, with $\xdt_t$ being a generic layer. Let $p(\xdt_T\,|\,\xdt_0)$ denote the conditional distribution of the network's output given the input for a neural network at initialization. To enforce desirable properties on $p(\xdt_T\,|\,\xdt_0)$, the strategy of \cite{pelux2020sde} consists in having the neural networks to converge, as $L$ goes to infinity, i.e. $\Dt \downarrow 0$, to a continuous-time stochastic process on the finite time interval $[0,T]$. In this case, for $L$ large enough, the conditional distribution $p(\xdt_T\,|\,\xdt_0)$ is close to the distribution of the limiting process at  $T$ given the same $\xdt_0$, with the limiting process being chosen in such a way to  make the transition density function well behaved. Among possible constructions of continuous-time stochastic processes as limits of discrete-time processes, here we consider the case where the limiting stochastic process has continuous paths. In all the neural network architectures considered in the present paper, each network's layer depends exclusively on the previous one, and hence $\xdt_t$ has the Markov property. These conditions identify diffusion processes \citep{stroock2006multidimensional} that are continuous-time Markov processes with continuous paths as natural candidates for the limiting process.

Let $\xdt_t$ denote a generic $D$-dimensional discrete-time Markov process, and let $\Delta \xdt_t = \xdt_{t+\Dt} - \xdt_t$ be the corresponding forward increments of the stochastic process. We recall assumptions that imply the convergence of $\xdt_t$ to a limiting diffusion process solution of a SDE. Such a SDE is referred to as the limiting SDE. In particular, it is implicit that the conditional distribution $p(\xdt_{t+\Dt}\,|\,\xdt_t)$ depends on $\Dt$ for the limiting diffusion process to exist as required.

\begin{assumption}[instantaneous mean function and covariance function] \label{ass:inf_coeff}
There exist a function $\mux(x): \R^D \rightarrow \R^D$ and a function $\sigmax^2(x): \R^D \rightarrow \R^{D \times D}$ such that for some $\delta > 0$ it holds
\begin{equation}
\lim_{\Dt \downarrow 0} \frac{\E[\Delta \xdt_t \,|\, \xdt_t]}{\Dt}   = \mux(\xdt_t)\label{eq:mu_x},
    \end{equation}
    \begin{equation}
\lim_{\Dt \downarrow 0} \frac{\V[\Delta \xdt_t \,|\, \xdt_t]}{\Dt}   = \sigmax^2(\xdt_t)\label{eq:sigma_x}
\end{equation}
and
\begin{equation}
 \lim_{\Dt \downarrow 0} \frac{\E[(\Delta \xdt_t)^{2 + \delta} \,|\, \xdt_t]}{\Dt} = 0\label{eq:continuity_x}
\end{equation}
uniformly on compact sets of $\mathbb{R}^D$ for each component, where $\mux(x)$ and $\sigmax^2(x)$ are continuous functions, and $\sigmax^2(x)$ is positive semi-definite, i.e. $\sigmax^2(x) = \sigmax(x) \sigmax(x)\tran$ for $\sigmax(x): \R^D \rightarrow \R^{D \times D}$.
\end{assumption}

The infinitesimal evolution of the diffusion processes considered in the present work is characterized by their instantaneous mean vector \cref{eq:mu_x} and instantaneous covariance matrix \cref{eq:sigma_x}. That is, the first two limits in \cref{ass:inf_coeff} pinpoint the form of the limiting SDE. The condition \cref{eq:continuity_x} represents a technical condition, in the sense that it allows us to consider the limits \cref{eq:mu_x} and \cref{eq:sigma_x} instead of their truncated version. We refer to \cite{nelson1990arch} for additional details on \cref{ass:inf_coeff} and related assumptions. The next theorem establishes that, under \cref{ass:inf_coeff} and additional assumptions, in the limit of $\xdt_t$ can be embedded in the solution of an SDE.

\begin{theorem} \label{thm:sde_convergence}
Under \cref{ass:inf_coeff}, extend the discrete-time stochastic process $\xdt_t$ to the stochastic process $\overline{\xdt}_t$ on $t \in [0,T]$ by continuous-on-right step-wise-constant interpolation of $\xdt_t$, i.e.
\begin{equation} \label{eq:sde_discrete_interpolated}
    \overline{\xdt}_t = \xdt_u\1_{u \leq t < u + \Dt}\qquad u \in \{0,\Dt,2\Dt,\dots,T\}.
\end{equation}
According to the construction \eqref{eq:sde_discrete_interpolated}, consider the $D$-dimensional SDE on $[0,T]$ with initial value $x_0 = \xdt_0$, drift vector $\mux(x)$ given by \cref{eq:mu_x}, and diffusion matrix $\sigmax(x)$ given by a square root of \cref{eq:sigma_x}:
\begin{equation} \label{eq:sde}
    dx_t = \mux(x_t) dt + \sigmax(x_t) dB_t,
\end{equation}
where $B_t$ is a $D$-dimensional Brownian motion (BM) with independent components, and \cref{eq:sde} means that
\begin{equation*}
    x_T = x_0 + \int_0^T \mux(x_t) dt + \int_0^T \sigmax(x_t) dB_t,
\end{equation*}
with the first and the second integral being a Riemann integral and an Ito integral, respectively. If the SDE \cref{eq:sde} admits a weak solution, and if this solution is unique in law and non-explosive, then the stochastic process $\overline{\xdt}_t$ defined in \cref{eq:sde_discrete_interpolated} converges in distribution to the solution of the SDE \cref{eq:sde}. This result still holds true for an independent and square integrable random variable $\xdt_0 \sim p(\xdt_0)$, provided that the driving BM is independent of $\xdt_0$. In both cases the convergence in distribution is on $\mathcal{D}([0,\infty),\mathbb{R}^D)$, the space of $\mathbb{R}^D$-valued processes on $[0,\infty)$ which are continuous from the right with finite left limits, endowed with the Skorohod metric \citep[Chapter 3]{billingsley1999convergence}.
\end{theorem}

We are dealing with three stochastic processes: i) the discrete-time stochastic process $\xdt_t$; ii) the continuous-time interpolation process $\overline{\xdt}_t$ of $\xdt_t$; iii) the limiting diffusion process $x_t$ of  $\xdt_t$. In Theorem \ref{thm:sde_convergence}, the continuous-time interpolation $\overline{\xdt}_t$ of $\xdt_t$ is introduced because we are seeking a continuous-time limiting process from a discrete-time stochastic process. The convergence established in Theorem \ref{thm:sde_convergence} is a strong convergence in the sense that it concerns with the convergence of the distribution of the stochastic process $(\overline{\xdt}_t)_{t \in [0,T]}$ as a stochastic object on the finite interval $[0,T]$ to the limiting diffusion process $(x_t)_{t \in [0,T]}$ as $L \uparrow \infty$. We consider weak solutions, as opposed to a strong solution, where it suffices that a BM $B_t$ can be found such that a solution can be obtained \citep[Section 5.3]{oksendal2003stochastic}. The focus on weak solutions and uniqueness in law of such a solutions, also known as weak uniqueness, is justified by our interest in the distributional properties of the limiting behavior of $\xdt_t$. In particular, weak solutions and uniqueness enable us to consider weaker requirements for attaining convergence of $\xdt_t$. In particular, consider the following discretization of the SDE \cref{eq:sde}:
\begin{equation} \label{eq:euler_sde}
    x_{t+\Dt} = x_t + \mux(x_t) \Dt + \sigmax(x_t) \zeta_t \sqrt{\Dt},
\end{equation}
where $\zeta_t$ is a $D$-dimensional random vector whose components are i.i.d. as standard Gaussian random variables (mean $0$ and variance $1$). Under suitable conditions, and in an appropriate sense \citep[Sections 10.2 and 14.1]{kloeden1992numerical}, it can be proved that the SDE \cref{eq:euler_sde} converges to the SDE \cref{eq:sde}. In the deterministic part of the SDE \cref{eq:euler_sde} we recognize the so-called Euler discretization of an ODE. In particular, \cref{thm:sde_convergence} postulates the existence and uniqueness in law of the weak solution of the limiting SDE, and its non-explosive behavior. This following assumptions state sufficient conditions for the postulated solution of the limiting SDE.

\begin{assumption}[weak solution and uniqueness in law on compact sets] \label{ass:existence_uniqueness}
The mean function $\mux(x)$ and the covariance function $\sigmax(x)$ are twice continuous and differentiable functions.
\end{assumption}

\begin{assumption}[non-explosive solution] \label{ass:non_explosivity}
There exists a finite $C > 0$ such that for each $x \in \R^{D}$:
\begin{equation*}
    \norm{\mux(x)} + \norm{\sigmax(x)} \leq C(1 + \norm{x}).
\end{equation*}
\end{assumption}

If \cref{ass:inf_coeff} and \cref{ass:existence_uniqueness} hold true and \cref{ass:non_explosivity} does not hold, then we still obtain convergence to the solution $x_{t}$ of the SDE \cref{eq:sde}. However, $x_t$ might diverge to infinity with positive probability on any time interval. We will return to this problem in the next section.

We conclude by introducing the main notation to be used throughout the paper: tensors (matrices, vectors) are indexed via subscripts ($u_i$, $u_{i,j}$, \dots), and we make use of $\bullet$ to index all elements and of $:$ to index ranges; we make no distinction between vectors and $n \times 1$ matrices, i.e. vectors are assumed to be column vectors; for a matrix $u$, $u\tran$ is its transpose, and if $u$ is square $\diag(u)$ is its diagonal vector and $\trace(u)$ is its trace; the norm of a vector $u$ is $\norm{u} = \sqrt{u\tran u}$; if $v$ is another vector their inner product is $\dprod{u}{v}=u \tran v$; the norm of a matrix $u$ is $\norm{u} = \sqrt{\trace(u \tran u)}$; for two matrices $u$ and $v$, $u v$ stands for the matrix multiplication, $u \otimes v$ for Kronecker's tensor product and $u \odot v$ for the element-wise product; we assume that matrix multiplication has higher precedence than element-wise product; for a tensor $u$, $\vect(u)$ is its vectorization (row-wise for matrices, with elements being traversed starting from the last dimension); we make use of $\I$ to denote the identity matrix and of $1$ to denote a vector of ones; for random variables $z$ and $w$, let $\var[z]$ be the variance of $z$, $\cov[z,w]$ be the covariance between $z$ and $w$ and $\rho[z,w]$ be their correlation; for random vectors $z \in \mathbb{R}^r$ and $w \in \mathbb{R}^c$ the $r \times c$ cross-covariance matrix $\C[z,w]$ is given by $\C[z,w]_{i,j} = \cov[z_i, w_j]$; the $r \times r$ covariance matrix of $z$ is thus $\V[z] = \C[z,z]$; the expectation $\E[z]$ of a random tensor $z$ is defined as the tensor of the expectations of its elements; for $D$-dimensional stochastic processes $z_t$ and $w_t$,  we make use of $[z]_t$ to denote the quadratic variation, which is a $D$-dimensional vector, and of $[z,w]_t$ to denote the quadratic covariation, which is a $D \times D$-dimensional matrix; for a differentiable function $f: \R^k \rightarrow \R$, $\nabla_u f(u) \in \R^k$ is its corresponding gradient vector; if $f: \R^k \rightarrow \R^m$, $J(f(u),u) \in \R^{m \times k}$ is its Jacobian matrix. We refer to \cref{tab:symb} for the complete notation.

\begin{table}
\caption{Notation: symbols and variables}
\label{tab:symb}
\centering
\begin{tabular}{lll}
\toprule
Symbol & Description \\
\midrule
$u_i$, $u_{i,j}$, $u_{i,j,l}$, \dots & vector, matrix, tensor indexing \\
& we use $:$ for ranges, $\bullet$ for all elements, $\patch$ for convolutional patches \\
$u\tran$ & matrix transpose \\
$\diag(u)$ & matrix diagonal \\
$\trace(u)$ & matrix trace \\
$\norm{u}$ & vector or matrix norm \\
$\dprod{u}{v}$ & inner product of 2 vectors \\
$u v$ & matrix multiplication \\
$u \otimes v$ & Kronecker's tensor product \\
$u \odot v$ & element-wise tensor product \\
$\vect(u)$ & tensor vectorization \\
$\I$ & identity matrix \\
$1$ & vector of ones \\
$\1$ & indicator function \\
$\var[z]$ & variance of a random variable \\
$\cov[z,w]$ & covariance between 2 random variables \\
$\rho[z,w]$ & correlation between 2 random variables \\
$\C[z,w]$ & cross-covariance matrix between 2 random vectors \\
$\V[z]$ & covariance matrix of a random vector \\
$\E[z]$ & expectation of a random tensor \\
$[z]_t$ & quadratic variation of a stochastic process \\
$[z,w]_t$ & quadratic covariation between 2 stochastic processes \\
$\nabla_u f(u)$ & function gradient \\
$J(f(u),u)$ & function jacobian \\
$\mathcal{N},\mathcal{MN},\mathcal{TN}$ & Gaussian, matrix-Gaussian, tensor-Gaussian distributions \\
\midrule
Variable & Description \\
\midrule
$D$ & width of each layer, indexed by $d$ \\
$L$ & number or layers, indexed by $l$ \\
$T$ & integration time, indexed by $t$ \\
$\Dt$ & time discretization interval \\
$\xdt_{l}$ & neural network layers \\
$\xdt_{t}$ & neural network layers over time \\
$\overline{\xdt}_t$ & interpolated neural network layers over time \\
$x_{t}$ & limit diffusion of $\overline{\xdt}_t$ (as $\Dt \downarrow 0$) \\
$\gdt_{t}$ & neural network output-input Jacobian over time \\
$\overline{\gdt}_t$ & interpolated neural network output-input Jacobian over time \\
$g_{t}$ & limit diffusion of $\overline{\gdt}_t$ (as $\Dt \downarrow 0$) \\
$\phi,\psi$ & activation functions \\
$A_t,a_t$ & neural network weights and biases \\
$W_t,b_t$ & neural network weights and biases diffusions \\
\bottomrule
\end{tabular}
\end{table}

\section{Infinitely deep ResNets}\label{sec:residual_diffusions}

\cite{pelux2020sde} investigated the implications of \cref{ass:inf_coeff}, \cref{ass:existence_uniqueness} and \cref{ass:non_explosivity} in the context of forward-propagation analysis of fully-connected ResNets, thus establishing an interplay between infinitely deep ResNets and solutions of SDEs, that is diffusion processes. In this section, we review and discuss the main results established in \cite{pelux2020sde}, and we extend them to the more general setting of convolutional ResNets. 

\subsection{Fully-connected ResNets}\label{sec:residual_diffusions_sub1}

We consider unmodified, albeit simplified, standard neural network architectures, which is in line with the research area of forward-propagation analysis \citep{poole2016exponential,schoenholz2017deep,hayou2019impact}. Consistently with Section \ref{sec:preliminaries_diffusion}, we consider a discrete-time stochastic process $\xdt_t \in \R^D$, which is assumed to be of constant dimensionality in order that $\Delta \xdt_t$ is well-defined. For \cref{ass:inf_coeff} to hold, we need $\Pr(\norm{\Delta \xdt_t}>\varepsilon\,|\,\xdt_t) \downarrow 0$ as $\Dt \downarrow 0$ for any $\varepsilon > 0$, that is we require the increments of the stochastic process $\xdt_t $ to vanish eventually. Intuitively, this is because the continuity of the paths of the limiting diffusion process.  According to the stochastic process $\xdt_t $, a fully-connected feedforward neural network is expressed by
\begin{displaymath}
\xdt_{t+\Dt} = f_t(\xdt_t) = \phi(A_t \xdt_t + a_t),
\end{displaymath}
for a nonlinear activation function $\phi: \R \rightarrow \R$ applied element-wise. We refer to $A_t \in \R^{D \times D}$ as weights and to $a_t \in \R^D$ as biases. Hence $\Delta \xdt_t = \phi(A_t \xdt_t + a_t) - \xdt_t$. Shrinking increments would imply that for all $x$, $\phi(A_t x + a_t)$ can be made arbitrarily concentrated around $x$ with a suitable choice of distributions for $(A_t,a_t)$. This cannot be achieved unless $\phi$ is linear or the distribution of $(A_t,a_t)$ depends on $x$. Indeed, fixing $x$ determines the values around which $(A_t,a_t)$ need to concentrate for the increments to vanish (if any), hence the increments will not vanish for a different $x' \neq x$, a fact that is most easily seen in the specific case where $(A_t,a_t)$ are scalars. 

The same lines of reasoning rules out the residual network architecture (ResNet), originally introduced in the work of \cite{he2016deep}. In particular, in the ResNet architecture we write $\xdt_{t+\Dt} = f_t(\xdt_t + r_t(\xdt_t))$. This leaves us with the identity ResNet of \cite{he2016identity} where we write
\begin{equation}\label{eq:shallow_block11}
    \xdt_{t+\Dt} = \xdt_t + r_t(\xdt_t)
\end{equation}
for some choice of $r_t$, the residual blocks, which we require to eventually vanish. Each $r_t$ results from an interleaved application of affine transforms and non-linear activation functions. \cite{pelux2020sde} considered the case of shallow residual blocks, such that \eqref{eq:shallow_block11} becomes
\begin{equation}\label{eq:shallow_block}
    \xdt_{t+\Dt} = \xdt_t + \phi(A_t \psi(\xdt_t) +  a_t)
\end{equation}
for two activation functions $\phi: \R \rightarrow \R$, $\psi: \R \rightarrow \R$ which are applied element-wise. We remark that the non-standard approach of using of 2 activation functions, i.e. $\phi$, $\psi$, is applied to cover the case of shallow residual blocks in full generality. For a shallow residual block $r_t$, the vanishing increments requirement is satisfied by having the distributions of weights $A_t$ and biases $a_t$ both concentrate around 0 provided that $\phi(0) = 0$. Furthermore, it proves to be advantageous to consider weights and biases given by increments of diffusion processes corresponding to solvable SDEs. Notice that the use of increments implies independence across layers, and the simplest parametrization corresponds to typical fully i.i.d. initializations used in practice.

\begin{assumption}[distributions of network's parameters and scaling] \label{ass:diffusion_parameters}
Let $W_t$ be a diffusion process in $\R^{D \times D}$ and let $b_t$ be a diffusion process in $\R^{D}$, which are defined as the solutions of
\begin{equation}\label{eq:param_diff_W}
dW_t = \mu^W dt + d\widetilde{W}_t
\end{equation}
with $d\vect(\widetilde{W}_t) = \sigma^W d\vect(B^W_t)$, and
\begin{equation}\label{eq:param_diffusion_b}
db_t = \mu^b dt + \sigma^b dB^b_t,
\end{equation}
respectively, where $B^W_t\in\R^{D \times D}$ and $B^b_t\in\R^{D}$ are independent BMs with independent components, $\mu^b \in \R^D$, $\sigma^W \in \R^{D^2 \times D^2}, \sigma^b \in \R^{D \times D}$, and $\Sigma^W = \sigma^W {\sigma^W}\tran$, $\Sigma^b = \sigma^b {\sigma^b} \tran$ are positive semi-definite. That is, $W_t$ and $b_t$ are matrix-valued and vector-valued diffusion processes, respectively, solutions of SDEs with deterministic time-homogeneous drift and diffusion coefficients.
\end{assumption}

Now, we consider the setting of \cref{ass:diffusion_parameters}. In particular, the discretizations of the diffusion processes $W_t$ and $b_t$ displayed in \eqref{eq:param_diff_W} and \eqref{eq:param_diffusion_b}, respectively, admit exact representations as
\begin{displaymath}
\Delta W_t = \mu^W \Dt + \varepsilon^W_t \sqrt{\Dt}
\end{displaymath}
and
\begin{displaymath}
   \Delta b_t = \mu^b \Dt + \varepsilon^b_t \sqrt{\Dt},
\end{displaymath}
where $\vect(\varepsilon^W_t) \overset{i.i.d.}{\sim} \mathcal{N}_{D^2}\big(0, \Sigma^W\big)$ and $\varepsilon^b_t \overset{i.i.d.}{\sim} \mathcal{N}_{D}\big(0, \Sigma^b\big)$ 
for $t=\Dt,\dots,T$, with $\mathcal{N}$ being the multivariate Gaussian distribution. According to the above discretizations of $W_t$ and $b_t$, we consider residual blocks where $A_t = \Delta W_t$ and $a_t = \Delta b_t$, and then we write the ResNet as follows
\begin{equation}
    \xdt_{t+\Dt} = \xdt_t + \phi(\Delta W_t \psi(\xdt_t) + \Delta b_t).\label{eq:resnet_fc}
\end{equation}
 \cref{ass:diffusion_parameters} covers the case where network's parameters are i.i.d. across layers according to a multivariate Gaussian distribution, up to the required scaling which is necessary to obtain the desired limiting diffusion process. By considering deterministic but time-dependent $\mu^W_t,\mu^b_t,\Sigma^W_t,\Sigma^b_t$ the extension to layer-dependent distributions is immediate. More generally, we can consider $W_t$ and $b_t$ driven by arbitrary SDEs. Moreover, dependencies across the parameters of different layers can be accommodated by introducing additional SDE-driven processes, driving the evolution of $W_t$ and $b_t$. We do not pursue further these directions in the present work. As for the activation functions, we require the following regularity assumptions.

\begin{assumption}[regularity of activation functions] \label{ass:activation}
The activation function $\psi: \R \rightarrow \R$ is continuously differentiable two times on $\R$. The activation function $\phi: \R \rightarrow \R$ is such that $\phi(0) = 0$ and, moreover, $\phi$ is continuously differentiable three times on $\mathbb{R}$ and its second and third derivatives have at most exponential tails growth, that is for some $k > 0$ it holds true
\begin{displaymath}
    \lim_{|x| \uparrow \infty} \frac{|\phi''(x)|}{e^{k |x|}} + \lim_{|x| \uparrow \infty} \frac{|\phi'''(x)|}{e^{k |x|}} < \infty.
\end{displaymath}
\end{assumption}

The assumption $\phi(0)=0$ and the smoothness assumptions on $\phi$ have been shown to be key requirements to achieve good signal propagation \citep{hayou2019impact,hayou2019training}. On the basis of \cref{ass:diffusion_parameters} and \cref{ass:activation}, we now report the main result of \cite{pelux2020sde}.

\begin{theorem} \label{thm:resnet_fc_sde}
Under \cref{ass:diffusion_parameters} and \cref{ass:activation}, \cref{ass:inf_coeff} with $\delta = 2$ holds true for the ResNet $\xdt_t$ defined in \cref{eq:resnet_fc}. The infinitesimal mean $\mu(x) $and covariance $\sigmax^2(x)$ functions are such that
\begin{displaymath}
\mux(x) = \phi'(0) (\mu^b + \mu^W \psi(x)) + \frac{1}{2} \phi''(0) \diag(\V[\varepsilon^W_t \psi(x) + \varepsilon^b_t\,|\,x])
    \end{displaymath}
    and
\begin{displaymath}    
\sigmax^2(x) = \phi'(0)^2 \V[\varepsilon^W_t \psi(x) + \varepsilon^b_t\,\|,x].
\end{displaymath}
Furthermore, \cref{ass:existence_uniqueness} is satisfied and, by means of \cref{thm:sde_convergence}, the continuous-time interpolation $\overline{\xdt}_t$ of the ResNet $\xdt_t$ converges in distribution to the solution on $[0,T]$ of the SDE
\begin{align} 
    &dx_t = \phi'(0) \big(\V[\varepsilon^W_t \psi(x_t) + \varepsilon^b_t\,|\,x_t]\big)^{1/2} dB_t\label{eq:resnet_fc_sde}\\
    &\quad\quad\quad+ \big(\phi'(0) (\mu^b + \mu^W \psi(x_t)) + \frac{1}{2} \phi''(0) \diag(\V[\varepsilon^W_t \psi(x_t) + \varepsilon^b_t\,|\,x_t])\big)dt\nonumber
\end{align}
with initial value $x_0 = \xdt_0$ where $B_t$ is a $D$-dimensional BM vector with independent components.
\end{theorem}

Theorem \ref{thm:resnet_fc_sde} does not establish a direct connection between the limiting diffusion process $x_t$ and the driving sources of randomness provided by the diffusion processes $W_t$ and $b_t$. Since we are interested in the study of properties of ResNets in the function space, that is over multiple input signals, a brute force approach would require to establish limiting diffusion processes as in \cref{thm:resnet_fc_sde} for an enlarged discrete-time stochastic process $\xdt_t=[\xdt_t^{(1)} \cdots \xdt_t^{(N)}] \in \R^{DN}$ corresponding to a collection of $N$ initial values $\xdt_0=[\xdt_0^{(1)} \cdots \xdt_0^{(N)}]$. Instead, \cite{pelux2020sde} showed that the limiting SDE is equivalent in law to the solution of another SDE which preserves the dependency on the driving sources of randomness. From here on, let $\xdt_t^{(i)}$ and $\xdt_t^{(j)}$ denote ResNets corresponding to two initial values $\xdt_0^{(i)}$ and $\xdt_0^{(j)}$, respectively. Moreover, let $x_t^{(i)}$ and $x_t^{(j)}$ denote limiting diffusion processes corresponding to the same two initial values, i.e. $x_0^{(i)} = \xdt_0^{(i)}$ and $x_0^{(j)} = \xdt_0^{(j)}$ respectively. Hereafter, we  continue to make use of $\xdt_t$ for $\xdt_t^{(i)}$ and $x_t$ for $x_t^{(i)}$ when no confusion arises. Under \cref{ass:diffusion_parameters} and \cref{ass:activation}, the next corollary characterizes the limiting diffusion process of the ResNet $\xdt_t^{(i)}$ 

\begin{corollary} \label{thm:resnet_fc_sde_2}
Let $\xdt_t^{(i)}$ be the ResNet corresponding to the initial value $\xdt_0^{(i)}$. Under \cref{ass:diffusion_parameters} and \cref{ass:activation}, the limiting diffusion process $x_t^{(i)}$ of $\xdt_t^{(i)}$ is the solution on $[0,T]$ of the SDE
\begin{equation} \label{eq:resnet_fc_sde_2}
    dx_t^{(i)} = \phi'(0)(dW_t \psi(x_t^{(i)}) + db_t) + \frac{1}{2}\phi''(0)(d[W \psi(x^{(i)})]_t + d[b]_t),
\end{equation}
and
\begin{equation} \label{eq:resnet_fc_sde_2_cross}
    d[x^{(i)},x^{(j)}]_t = \phi'(0)^2(d[W\psi(x^{(i)}),W\psi(x^{(j)})]_t + d[b,b]_t).
\end{equation}
\end{corollary}

A direct consequence of \cref{thm:resnet_fc_sde} is that the distribution of the ResNet output given the input, i.e. $p(\xdt_T\,|\,\xdt_0)$, converges to the transition density $p(x_T\,|\,x_0)$ of the solution of \cref{eq:resnet_fc_sde_2}. In particular, as $T$ is finite, the dependency on the input does not vanish in the limit of infinite total depth $L$ and can be controlled via the distributions of network's parameters and the integration time $T$. The representations  \cref{eq:resnet_fc_sde} and \cref{eq:resnet_fc_sde_2} are complementary: depending on the situation it will prove advantageous to use one or the other. Theorem \ref{thm:resnet_fc_sde} and Corollary \ref{thm:resnet_fc_sde_2} are general in the sense that we allow for an arbitrary covariance structure between the elements of $\varepsilon^W_t$, i.e. an arbitrary (constant and deterministic) quadratic covariation for $W_t$. This makes it difficult to derive more explicit results, and is also an impractical approach as the parametrization requires $\mathcal{O}(D^4)$ elements. \cite{pelux2020sde} then considered more restrictive distribution assumptions with a more manageable $\mathcal{O}(D^2)$ parametrization cost.

\begin{assumption}[matrix Gaussian network's parameters] \label{ass:diffusion_parameters_matrix_variate}
Let $b_t,\mu^b,\sigma_b,B^b_t,\mu^W,B^W_t$ be defined as in \cref{ass:diffusion_parameters}, and let $W_t$ be a diffusion process in $\R^{D \times D}$, which is defined as the solution of 
\begin{displaymath}
    dW_t = \mu^W dt + \sigma^{W_O} dB^W_t \sigma^{W_I},
\end{displaymath}
where $\sigma^{W_O},\sigma^{W_I} \in \R^{D \times D}$ and $\Sigma^{W_O} = \sigma^{W_O} {\sigma^{W_O}} \tran$, $\Sigma^{W_I} = {\sigma^{W_I} }\tran \sigma^{W_I}$ are positive semi-definite matrices.
\end{assumption}

We consider the setting of  \cref{ass:diffusion_parameters_matrix_variate}. Under this setting the discretization of $W_t$ satisfies
\begin{displaymath}
\varepsilon^W_t \overset{i.i.d.}{\sim} \mathcal{MN}_{D,D}\big(0, \Sigma^{W_O}, \Sigma^{W_I} \big)
\end{displaymath}
for $t=\Dt,\dots,T$, where $\mathcal{MN}$ stands for the matrix Gaussian distribution. This is a direct consequence of the following fact: if $\zeta \sim \mathcal{MN}(0, \I, \I)$ then $A \zeta B \sim \mathcal{MN}(0, A A\tran, B\tran B)$. The fundamental property of  the $\mathcal{MN}$ distributions is that the covariance factorizes as $\cov(\varepsilon^W_{o,i},\varepsilon^W_{o',i'}) = \Sigma^{W_O}_{o,o'}\Sigma^{W_I}_{i,i'}$. The reader is referred to \cite{gupta1999matrix} for a comprehensive treatment of matrix Gaussian distributions and their properties. Under  \cref{ass:activation} and \cref{ass:diffusion_parameters_matrix_variate}, the next corollary characterizes the limiting diffusion process of the ResNet $\xdt_t^{(i)}$.

\begin{corollary} \label{thm:resnet_fc_matrix_variate}
 Let $\xdt_t^{(i)}$ be the ResNet corresponding to the initial value $\xdt_0^{(i)}$. Under \cref{ass:activation} and \cref{ass:diffusion_parameters_matrix_variate}, the limiting diffusion process $x_t^{(i)}$ of $\xdt_t^{(i)}$ is the solution on $[0,T]$ of the SDE
\begin{align}
    &dx_t^{(i)} = \phi'(0)\big((\mu^W \psi(x_t^{(i)}) + \mu^b) dt + \sigma^{W_O} dB^W_t \sigma^{W_I} \psi(x_t^{(i)}) + \sigma^b dB^b_t\big)\label{eq:resnet_fc_matrix_variate}\\ 
    &\quad\quad\quad+ \frac{1}{2}\phi''(0)\diag\big(\Sigma^b + \Sigma^{W_O} (\psi(x_t^{(i)})\tran \Sigma^{W_I} \psi(x_t^{(i)})) \big)dt\nonumber,
    \end{align}
  and  
    \begin{displaymath}
d[x^{(i)},x^{(j)}]_t = \phi'(0)^2\big(\Sigma^b + \Sigma^{W_O} \psi(x^{(i)}_t)\tran \Sigma^{W_I} \psi(x^{(j)}_t)\big)dt.
    \end{displaymath}
\end{corollary}

Under \cref{ass:diffusion_parameters_matrix_variate} we have that the covariance function $\V[\varepsilon^W_t \psi(x_t) + \varepsilon^b_t\,|\,x_t]$ is given by $\Sigma^b + \Sigma^{W_O} (\psi(x_t)\tran \Sigma^{W_I} \psi(x_t))$. In particular, the dependency on the state $x_t$ in Equation \cref{eq:resnet_fc_sde} goes through a linear transformation and a weighted inner product. This fact sheds light on the impact of introducing dependencies among row and columns of network's parameters $A_t = \Delta W_t$. Specifically, the matrix $\Sigma^{W_I}$ defines the structure of the inner weighted product, while the matrix $\Sigma^{W_O}$ defines how such transforms affect each dimension $d \in D$. \cite{pelux2020sde} completed their study by considering the simplest fully i.i.d. setting with the assumption of centered distributions for $W_t$ and $b_t$. Fully i.i.d. initializations are commonly used in training of neural networks. A scaling of the weights by $D^{-1/2}$ is also introduced; this is the same scaling used to obtain Gaussian process limits in infinitely wide networks \citep{neal1995bayesian,lee2018deep}. In \cref{sec:doubly_infinite} we show that such a scaling allows to study $D \uparrow \infty$.

\begin{assumption}[fully i.i.d. network's parameters] \label{ass:diffusion_parameters_iid}
Let $W_t$ be a diffusion process in $\R^{D \times D}$ and let $b_t$ be a diffusion process in $\R^{D}$, which are defined as the solutions of the SDSs
\begin{equation}\label{eq:param_diff_W_iid}
dW_t = \frac{\sigma_w}{\sqrt{D}} dB^W_t
\end{equation}
and
\begin{equation}\label{eq:param_diffusion_b_iid}
db_t = \sigma_b dB^b_t,
\end{equation}
respectively, where $B^W_t\in\R^{D \times D}$ and $B^b_t\in\R^D$ are independent BMs  and $\sigma_w>0,\sigma_b>0$ are scalars.
\end{assumption}

Now, we consider the setting of \cref{ass:diffusion_parameters_iid}. In particular, the discretizations of the diffusion processes $W_t$ and $b_t$ displayed in \eqref{eq:param_diff_W_iid} and \eqref{eq:param_diffusion_b_iid}, respectively, admit exact representations as
\begin{equation}\label{eq:fc_iid_discr_std1}
\Delta W_t = \varepsilon^W_t \frac{\sigma_w}{\sqrt{D}} \sqrt{\Dt}
 \end{equation}  
 and
 \begin{equation}\label{eq:fc_iid_discr_std2}
\Delta b_t = \varepsilon^b_t \sigma_b \sqrt{\Dt}
 \end{equation} 
where $\varepsilon^W_t \overset{i.i.d.}{\sim} \mathcal{MN}_{D,D}\big(0, \I_D, \I_D \big)$ and $\varepsilon^b_t \overset{i.i.d.}{\sim} \mathcal{N}_{D}\big(0, \I_D \big)$. Under \cref{ass:activation} and \cref{ass:diffusion_parameters_iid}, the next corollary characterizes the limiting diffusion process of the ResNet $\xdt_t^{(i)}$.

\begin{corollary} \label{thm:resnet_fc_iid}
et $\xdt_t^{(i)}$ be the ResNet corresponding to the initial value $\xdt_0^{(i)}$. Under \cref{ass:activation} and \cref{ass:diffusion_parameters_iid}, the limiting diffusion process $x_t^{(i)}$ of $\xdt_t^{(i)}$ is the solution on $[0,T]$ of the SDE
\begin{align}
    &dx_t^{(i)} = \phi'(0)\big(\frac{\sigma_w}{\sqrt{D}} dB^W_t \psi(x_t^{(i)}) + \sigma_b dB^b_t\big)\label{eq:resnet_fc_iid}\\ 
    &\notag\quad\quad\quad+ \frac{1}{2}\phi''(0)\big(\sigma_b^2 + \frac{\sigma_w^2}{D} \norm{\psi(x_t^{(i)})}^2) \big) \I_D dt,
    \end{align}
    and
\begin{displaymath}
d[x^{(i)},x^{(j)}]_t = \phi'(0)^2\big(\sigma_b^2 + \frac{\sigma_w^2}{D} \dprod{\psi(x^{(i)}_t)}{\psi(x^{(j)}_t)}\big) \I_D dt.
\end{displaymath}
\end{corollary}

In all cases the activation function $\phi$ only impacts the dynamics through its local behavior at the origin, while this is not the case for the activation $\psi$. Under the setting of \cref{ass:diffusion_parameters_iid}, i.e. fully i.i.d. network's parameters, we have that $\V[\varepsilon^W_t \psi(x_t) + \varepsilon^b_t\,|\,x_t]$ is given by $\sigma_b^2 + \frac{\sigma_w^2}{D} \norm{\psi(x_t)}^2$. In particular, the dependency on the state $x_t$ in \cref{eq:resnet_fc_sde} goes only through the norm of $x_t$, which is permutation invariant in $d \in D$. Accordingly, the distribution of the stochastic processes $x_{t,d}$ is exchangeable across $d \in D$ if the distribution of $x_{0,d}$ is so. We will show in \cref{sec:doubly_infinite_diff} that, under \cref{ass:diffusion_parameters_iid}, as $D \uparrow \infty$ $x_{t,d}$ will become i.i.d. over $d$ if $x_{0,d}$ is so.

\begin{remark} According to Equation \cref{eq:resnet_fc_sde_2} and Equation \cref{eq:resnet_fc_sde_2_cross}, the joint evolution of the diffusion processes $x_t^{(i)}$ and $x_t^{(j)}$ corresponding to two inputs $x_0^{(i)}$ and $x_0^{(j)}$, respectively, is not perfectly correlated. This remains true also in the parameterizations of \cref{ass:diffusion_parameters_matrix_variate} and \cref{ass:diffusion_parameters_iid}. Thus in the limit of infinite total depth $L$ the distribution in function space does not suffer from the perfect correlation problem. The joint distribution $p(x_T^{(i)},x_T^{(j)}\,|\,x_0^{(i)},x_0^{(j)})$ is not Gaussian. We show in \cref{sec:doubly_infinite_diff} that we recover the Gaussian case as $D \uparrow \infty$ under the parametrization of \cref{ass:diffusion_parameters_iid}.
\end{remark}

\begin{remark} A standard time-change result for SDEs \citep[Theorem 8.5.7]{oksendal2003stochastic} implies that the time-scaling of an SDE is equivalent to multiplying the drift and the diffusion coefficients by the scaling constant and by the square root of the scaling constant, respectively. Furthermore, according to Equation \cref{eq:resnet_fc_sde}, it is possible to compensate changes in the integration time $T$ with changes in the ``hyper-parameters'' $\mu^b,\mu^W,\Sigma^b,\Sigma^W$ in \cref{ass:diffusion_parameters} to leave the dynamics of \cref{eq:resnet_fc_sde} invariant. These observations remain true also in the parameterizations of \cref{ass:diffusion_parameters_matrix_variate} and \cref{ass:diffusion_parameters_iid}. Therefore, without loss of generality, we can restrict to $T=1$.
\end{remark}

\begin{remark}\label{rem:explosion} Without further assumptions, the solutions to the limiting SDEs can be explosive solutions. \cref{ass:non_explosivity} is satisfied under all considered distributional assumptions for network's parameters if either: i) $\psi$ exhibits at most square-root growth, in particular $\psi$ is bounded; or ii) $\psi$ exhibits at most linear growth, in particular $\psi$ is the identity function, and $\phi''(0)=0$, in particular $\phi = \tanh$. We will show in \cref{sec:doubly_infinite_diff} that, under \cref{ass:diffusion_parameters_iid}, as $D \uparrow \infty$ in the case of $\phi''(0) \neq 0$ with $\psi$ the identity function, the explosion time becomes deterministic.
\end{remark}

\begin{remark} The limiting diffusion processes that we have obtained are based on smoothness assumptions for $\phi$. Given the popularity of the $\ReLU$ activation function $\phi(a) = \max(0, a)$, we consider here a brief analysis which includes it. In particular, we assume that $\phi(a)$ is positively homogeneous, i.e. $\phi(\alpha a) = \alpha \phi(a)$ for $\alpha > 0$, $h$ is random variable, and $\gamma > 0$ then: $\E[\phi(h \Dt^\gamma)/\Dt] = \E\left[\phi(h) \right] \Dt^{\gamma - 1}$ and $\E[\phi(h \Dt^\gamma)^2/\Dt] = \E\left[\phi(h)^2\right] \Dt^{2\gamma - 1}$. Comparing these results with \cref{eq:mu_x} and \cref{eq:sigma_x}, we see that unless $\E[\phi(h)] = 0$, the choice of $\gamma = 1/2$ would result in the drift term blowing up. The alternative of choosing $\gamma = 1$ recovers a non stochastic limit which can be interpreted as a particular form of \cite{chen2018neural}. The  positive homogeneity of $\ReLU$ activations makes equivalent to modify the recursion or reparameterize the parameter.
\end{remark}

We have considered $\xdt_0 \in \R^D$ to be the input of the ResNet. A neural network acts as a function approximator to be fitted to some dataset $\mathcal{D} = (\mathcal{Z},\mathcal{Y}) = \{(z^{(i)},y^{(i)})\}_{i=1}^N$ of size $N$ where $z^{(i)} \in \R^Z$ represents an input and $y^{(i)} \in \R^Y$ represents the corresponding output. Classification problems can be framed in this setting if we use a one-hot representation for $y^{(i)}$. In general, there can be a mismatch between $D,Z$ and $Y$, making it is necessary to introduce adaptation layers $z^{(i)} \mapsto \xdt_0^{(i)}$ and $\xdt_T^{(i)} \mapsto \widehat{y}^{(i)}$ where $\widehat{y}^{(i)}$ is the network prediction for $z^{(i)}$. As for $\xdt_t$, we denote a single data-point $(z^{(i)},y^{(i)})$ with $(z,y)$ when no confusion arises.

\subsection{Convolutional ResNets}\label{sec:cnn}

We extend the main results of \cite{pelux2020sde} to the more general setting of convolutional neural networks (CNNs). Such extension relies on the equivalence between convolutional transformations (either at a given position, or over all positions) and specific forms of matrix multiplication. For the sake of simplicity and clarity in exposition, we present our results for 2D convolutions and square filters. Analogous results follow in the more general setting. CNNs are best described by keeping the features, height and width dimensions separated, in which case $\xdt_t$ is a three-dimensional tensor. This does not cause issues to our analysis, since we can consider the vectorization $\vect(\xdt_t)$ which allows us to refer to definitions and results of \cref{sec:preliminaries_diffusion}. We denote the input image to the convolutional neural network and its layers with $\xdt_t$, $t=0,\Dt,\dots,T$. As before $\xdt_t$ needs to be of fixed dimensionality: $\xdt_t \in \R^{U \times V \times D}$, $D$ being the number of channels, and $U$ and $V$ being respectively the height and the width. 

We consider square filters of spatial length $K$, with $K$ being odd, in which case the off-center range of the filter is $E = (K - 1)/2$. Assuming unitary strides in both height and width dimensions, constant dimensionality is achieved by padding the width and height dimensions of $\xdt_t$, $t > 0$, with $E$ pixels borders. The padding can be performed arbitrarily here: typically the values next to the boarder are copied or paddings have the same value of a background reference level. We enumerate the set of $P = UV$ positions, where positions are ordered in row-wise manner (the ordering does not affect the results as long as it is the same everywhere). A convolutional transform $x \in \R^{(U \times V \times D)} \mapsto y \in \R^{(U \times V \times D)}$ is obtained by applying (convolving) the same filter $W \in \R^{D \times (U \times V \times D)}$ to the extracted patches $x_{\patch p} \in \R^{(U \times V \times D)}$ by matrix multiplication: $y_p = W x_{\patch p}$, $y_p \in \R^D$ for each $p=1,\dots,P$. Parentheses indicate how the dimensions are flattened (vectorized), and each patch is given by $x_{\patch p} = x_{\patch p,\bullet} = x_{u-E:u+E,v-E:v+E,\bullet}$ for position $p=(u,v)$, $u=1,\dots,U$, $v=1,\dots,V$. We incorporate the padding in the patch extraction operation: indexing outside the allowed ranges (which happens for positions at the boarders) returns the padded values. More generally a bias term $b \in \R^D$ can be included resulting in $y_p = W x_{\patch p} + b$. See  \cite{dumoulin2016guide} and references therein for a comprehensive account. For convenience let $F = UVD$ denote the extracted patch size. We begin with the most generic parametrization for CNNs covered in this work, which corresponds to \cref{ass:diffusion_parameters} for the fully-connected case.

\begin{assumption}[distributions of network's parameters and scaling] \label{ass:diffusion_parameters_cnn}
Let $W_t$ be a diffusion process in  $\R^{D \times K \times K \times D}$ and let $b_t$ be the diffusion process in $\R^{D}$, which are defined as the solutions of
\begin{equation}\label{eq:param_diff_W_cnn}
dW_t = \mu^W dt + d\widetilde{W}_t
\end{equation}
with $d\vect(\widetilde{W}_t) = \sigma^W d\vect(B^W_t)$, and 
\begin{equation}\label{eq:param_diffusion_b_cnn}
db_t = \mu^b dt + \sigma^b dB^b_t
\end{equation}
where $B^W_t\in\R^{D \times K \times K \times D}$ and $B^b_t\in\R^{D}$ are independent BMs with independent components, $\mu^W \in \R^{D \times K \times K \times D}, \mu^b \in \R^D$, $\sigma^W \in \R^{DF \times DF}, \sigma^b \in \R^{D \times D}$, and $\Sigma^W = \sigma^W {\sigma^W}\tran$, $\Sigma^b = \sigma^b {\sigma^b} \tran$ are positive semi-definite. That is, $W_t$ and $b_t$ are tensor-valued and vector-valued diffusion processes, respectively, solutions of SDEs with deterministic time-homogeneous  drift and diffusion coefficients.
\end{assumption}

Now, we consider the setting of \cref{ass:diffusion_parameters_cnn}. In particular, the discretizations of the diffusion processes $W_t$ and $b_t$ displayed in \eqref{eq:param_diff_W_cnn} and \eqref{eq:param_diffusion_b_cnn}, respectively, admit exact representations as
\begin{displaymath}
\Delta W_t = \mu^W \Dt + \varepsilon^W_t \sqrt{\Dt}
\end{displaymath}
and
\begin{displaymath}
\Delta b_t = \mu^b \Dt + \varepsilon^b_t \sqrt{\Dt},
\end{displaymath}
where $\vect(\varepsilon^W_t) \overset{i.i.d.}{\sim} \mathcal{N}_{DF}\big(0, \Sigma^W\big)$ and $\varepsilon^b_t \overset{i.i.d.}{\sim} \mathcal{N}_{D}\big(0, \Sigma^b\big)$ for $t=\Dt,\dots,T$ where $\mathcal{N}$ stands for the multivariate Gaussian distribution. As in \cite{pelux2020sde}, we consider shallow residual blocks and two activation functions. This leads to write the ResNet as follows
\begin{equation}\label{eq:resnet_cnn}
    \xdt_{t+\Dt,p} = \xdt_{t,p} + \phi(\Delta W_t \psi(\xdt_{t,\patch p}) + \Delta b_t)\qquad(p=1,\dots,P)
\end{equation}
where $\Delta W_t \psi(\xdt_{t,\patch p})$ is computed by means of matrix multiplication as we have explained above. The next theorem states our main convergence result in the setting of convolutional ResNets. In particular, the next theorem provides the convolutional counterpart of Theorem \ref{thm:resnet_fc_sde}.

\begin{theorem} \label{thm:resnet_cnn_sde}
Under \cref{ass:diffusion_parameters_cnn} and \cref{ass:activation}, \cref{ass:inf_coeff} with $\delta = 2$ holds true for the ResNet $\xdt_t$ defined in \cref{eq:resnet_cnn}. The infinitesimal mean $\mu(x)$ and covariance $\sigma^{2}(x)$ functions are such that
\begin{displaymath}
\mux(x) = \begin{bmatrix} \mu_{\patch}(x_{\patch 1}) \\ \vdots \\ \mu_{\patch}(x_{\patch P}) \end{bmatrix}
\end{displaymath}
where
\begin{displaymath}
\mu_{\patch}(x_{\patch p}) = \phi'(0) (\mu^b + \mu^W \psi(x_{\patch p})) + \frac{1}{2} \phi''(0) \diag(\V[\varepsilon^W_t \psi(x_{\patch p}) + \varepsilon^b_t\,|\,x_{\patch p}]),
\end{displaymath}
and
\begin{align*}
&\sigmax^2(x) = \phi'(0)^2 \begin{bmatrix} 
    \sigma^2_{\patch}(x_{\patch 1},x_{\patch 1}) & \cdots & \sigma^2_{\patch}(x_{\patch 1},x_{\patch P})\\
    \vdots & \vdots & \vdots \\
    \sigma^2_{\patch}(x_{\patch P},x_{\patch 1}) & \cdots & \sigma^2_{\patch}(x_{\patch P},x_{\patch P}) \end{bmatrix}
    \end{align*}
where
    \begin{displaymath}
\sigma^2_{\patch}(x_{\patch p},x_{\patch p'}) = \C[\varepsilon^W_t \psi(x_{\patch p}) + \varepsilon^b_t, \varepsilon^W_t \psi(x_{\patch p'}) + \varepsilon^b_t\,|\,x_{\patch p},x_{\patch p'}].
\end{displaymath}
Furthermore, \cref{ass:existence_uniqueness} is satisfied and, by means of  \cref{thm:sde_convergence}, the continuous-time interpolation $\vect(\overline{\xdt_t})$ of the ResNet $\vect(\xdt_t)$ converges in distribution to the solution on $[0,T]$ of the SDE
\begin{align}
&d\vect(x_t) \label{eq:resnet_cnn_sde}\\
&\quad= \phi'(0) \begin{bmatrix} 
    \sigma^2_{\patch}(x_{t,\patch 1},x_{t,\patch 1}) & \cdots & \sigma^2_{\patch}(x_{t,\patch 1},x_{t,\patch P})\\
    \vdots & \vdots & \vdots \\
    \sigma^2_{\patch}(x_{t,\patch P},x_{t,\patch 1}) & \cdots & \sigma^2_{\patch}(x_{t,\patch P},x_{t,\patch P}) 
\end{bmatrix}^{1/2} dB_t +\begin{bmatrix} \mu_{\patch}(x_{t,\patch 1}) \\ \vdots \\ \mu_{\patch}(x_{t,\patch P}) \end{bmatrix} dt\nonumber
\end{align}
with initial value $x_0 = \xdt_0$ where $B_t$ is a $PD$-dimensional BM vector with independent components.
\end{theorem}

The proof of Theorem \ref{thm:resnet_cnn_sde} is omitted. This is because the proof is obtained along lines similar to the proof of Theorem \ref{thm:resnet_fc_sde}, while being more cumbersome due to the extra spacial dimensions. Notice that the dimensionality of the driving Brownian motion depends on $U,V$. As in \cref{sec:residual_diffusions_sub1} we can restate \cref{thm:resnet_cnn_sde} by making explicit the dependency on the driving sources of randomness. In particular, this allows us to formulate the dynamics of $\xdt_t$ as integration with respect to Brownian motions whose dimensionality does not depend on the number of inputs, nor their spatial sizes $U,V$. Under \cref{ass:diffusion_parameters_cnn} and \cref{ass:activation}, the next corollary characterizes the limiting diffusion process of the convolutional RenNet at the initial value $\xdt_{0}^{(i)}$.

\begin{corollary} \label{thm:resnet_cnn_sde_2}
Let $\xdt_{t,p}^{(i)}$ be the ResNet corresponding to the initial value $\xdt_{0}^{(i)}$. Under \cref{ass:diffusion_parameters_cnn} and \cref{ass:activation}, the limiting diffusion process $x_{t,p}^{(i)}$ of $\xdt_{t,p}^{(i)}$ is the solution on $[0,T]$ of the SDE
\begin{equation} \label{eq:resnet_cnn_sde_2}
dx_{t,p}^{(i)} = \phi'(0)(dW_t \psi(x_{t,\patch p}^{(i)}) + db_t) + \frac{1}{2}\phi''(0)(d[W \psi(x_{\patch p}^{(i)})]_t + d[b]_t)
\end{equation}
for $p=1,\dots,P$, and 
\begin{equation} \label{eq:resnet_cnn_sde_2_cross}
d[x_{p}^{(i)},x_{p'}^{(j)}]_t = \phi'(0)^2(d[W\psi(x_{\patch p}^{(i)}),W\psi(x_{\patch p'}^{(j)})]_t + d[b,b]_t)
\end{equation}
\end{corollary}

The parametrization of \cref{ass:diffusion_parameters_cnn} is $\mathcal{O}(D^2F^2)$. Hereafter we introduce a more parsimonious parameterization which is based on tensor Gaussian distributions; this is a natural generalization of the matrix Gaussian distribution in \cite{gupta1999matrix}. The use of Kronecker products allows us to cover this parametrization with a compact notation. We also introduce a fully i.i.d. initialization with the same scaling with $D$ as in the fully-connected case.

\begin{assumption}[tensor Gaussian network's parameters] \label{ass:diffusion_parameters_tensor_variate}
Let $b_t,\mu^b,\sigma_b,B^b_t,\mu^W,B^W_t$ be defined as in \cref{ass:diffusion_parameters_cnn}, and let $W_t$ be a diffusion process in $\R^{D \times (K \times K \times D)}$, which is defined as the solution of
\begin{equation*}
dW_t = \mu^W dt + \sigma^{W_O} dB^W_t (\sigma^{W_U} \otimes \sigma^{W_V} \otimes \sigma^{W_I}),
\end{equation*}
where $\sigma^{W_O},\sigma^{W_I} \in \R^{D \times D}$, $\sigma^{W_U},\sigma^{W_V} \in \R^{K \times K}$ and $\Sigma^{W_O} = \sigma^{W_O} {\sigma^{W_O}} \tran$, $\Sigma^{W_U} = \sigma^{W_U} {\sigma^{W_U}} \tran$, $\Sigma^{W_V} = \sigma^{W_V} {\sigma^{W_V}} \tran$, $\Sigma^{W_I} = {\sigma^{W_I} }\tran \sigma^{W_I}$ are positive semi-definite matrices
\end{assumption}

We consider the setting of  \cref{ass:diffusion_parameters_tensor_variate}. Under this setting the discretization of $W_t$ satisfies:
\begin{equation*}
    \varepsilon^W_t \overset{i.i.d.}{\sim} \mathcal{TN}_{D,K,K,D}\big(0, \Sigma^{W_O}, \Sigma^{W_U}, \Sigma^{W_V}, \Sigma^{W_I} \big)
\end{equation*}
for $t=\Dt,\dots,T$, where $\mathcal{TN}$ stands for the tensor Gaussian distribution, and we have $\cov(\varepsilon^W_{o,u,v,i},\varepsilon^W_{o',u',v',i'}) = \Sigma^{W_O}_{o,o'}\Sigma^{W_U}_{u,u'}\Sigma^{W_V}_{v,v'}\Sigma^{W_I}_{i,i'}$. Under \cref{ass:diffusion_parameters_tensor_variate} and \cref{ass:activation}, the next corollary characterizes the limiting diffusion process of the convolutional ResNet at the initial value $\xdt_{0}^{(i)}$. In particular, the next corollary provides the convolutional counterpart of Corollary \ref{thm:resnet_fc_matrix_variate}.

\begin{corollary} \label{thm:resnet_cnn_matrix_variate}
Let $\xdt_{t,p}^{(i)}$ be the ResNet corresponding to the initial value $\xdt_{0}^{(i)}$. Under \cref{ass:diffusion_parameters_tensor_variate} and \cref{ass:activation}, the limiting diffusion process $x_{t,p}^{(i)}$ of $\xdt_{t,p}^{(i)}$ is the solution on $[0,T]$ of the SDE
\begin{align}
&dx_{t,p}^{(i)} = \phi'(0)\big((\mu^W \psi(x_{t,\patch p}^{(i)}) + \mu^b) dt\label{eq:resnet_cnn_matrix_variate}\\
&\quad\quad\quad+ \sigma^{W_O} dB^W_t (\sigma^{W_U} \otimes \sigma^{W_V} \otimes \sigma^{W_I}) \psi(x_{t, \patch p}^{(i)}) + \sigma^b dB^b_t\big)\nonumber\\ 
&\quad\quad\quad\quad+ \frac{1}{2}\phi''(0)\diag\big(\Sigma^b + \Sigma^{W_O} (\psi(x_{t,\patch p}^{(i)})\tran (\Sigma^{W_U} \otimes \Sigma^{W_V} \otimes \Sigma^{W_I}) \psi(x_{t, \patch p}^{(i)})) \big)dt\nonumber
\end{align}
and
\begin{displaymath}
d[x_{p}^{(i)},x_{p'}^{(j)}]_t = \phi'(0)^2\big(\Sigma^b + \Sigma^{W_O} \psi(x^{(i)}_{t,\patch p})\tran (\Sigma^{W_U} \otimes \Sigma^{W_V} \otimes \Sigma^{W_I}) \psi(x^{(j)}_{t, \patch p'})\big)dt.
\end{displaymath}
\end{corollary}

\begin{assumption}[fully i.i.d. network's parameters] \label{ass:diffusion_parameters_iid_cnn}
Let $W_t$  be a diffusion process in $\R^{D \times K \times K \times D}$ and let $b_t$ be the diffusion process in $\R^{D}$, which are defined as solutions of the SDEs
\begin{equation}\label{eq:param_diff_W_iid_cnn}
dW_t = \frac{\sigma_w}{\sqrt{D}} dB^W_t
 \end{equation}
 and
 \begin{equation}\label{eq:param_diffusion_b_iid_cnn}
db_t = \sigma_b dB^b_t,
\end{equation}
respectively, where $B^W_t\in\R^{D \times D}$ and $ B^b_t\in\R^D$ are independent BMs and $\sigma_w>0,\sigma_b>0$ are scalars.
\end{assumption}

Now, we consider the setting of \cref{ass:diffusion_parameters_iid_cnn}. In particular, the discretizations of the diffusion processes $W_t$ and $b_t$ displayed in \eqref{eq:param_diff_W_iid_cnn} and \eqref{eq:param_diffusion_b_iid_cnn}, respectively, admit exact representations as
\begin{displaymath}
\Delta W_t = \zeta^W_t \frac{\sigma_w}{\sqrt{D}} \sqrt{\Dt}
\end{displaymath}
and
\begin{displaymath}
\Delta b_t = \zeta^b_t \sigma_b \sqrt{\Dt},
\end{displaymath}
where $\zeta^W_t \overset{i.i.d.}{\sim} \mathcal{TN}_{D,K,K,D}\big(0, \I_D, \I_K, \I_K, \I_D \big)$ and $\zeta^b_t \overset{i.i.d.}{\sim} \mathcal{N}_{D}\big(0, \I_D \big)$. As for the setting of fully-connected neural networks, fully i.i.d. initializations are commonly used in the context of the training of convolutional neural networks. Under \cref{ass:diffusion_parameters_iid_cnn} and \cref{ass:activation}, the next corollary characterizes the limiting diffusion process of the convolutional ResNet at the initial value $\xdt_{0}^{(i)}$. In particular, the next corollary provides the convolutional counterpart of Corollary \ref{thm:resnet_fc_iid}.

\begin{corollary} \label{thm:resnet_cnn_iid}
Let $\xdt_{t,p}^{(i)}$ be the ResNet corresponding to the initial value $\xdt_{0}^{(i)}$. Under \cref{ass:diffusion_parameters_iid_cnn} and \cref{ass:activation}, the limiting diffusion process $x_{t,p}^{(i)}$ of $\xdt_{t,p}^{(i)}$ is the solution on $[0,T]$ of the SDE
\begin{align}\label{eq:resnet_cnn_iid}
    &dx_{t,p}^{(i)} = \phi'(0)\big(\frac{\sigma_w}{\sqrt{D}} dB^W_t \psi(x_{t,\patch p}^{(i)}) + \sigma_b dB^b_t\big)\\ 
    &\quad\quad\quad+ \frac{1}{2}\phi''(0)\big(\sigma_b^2 + \frac{\sigma_w^2}{D} \norm{\psi(x_{t,\patch p}^{(i)})}^2) \big) \I_D dt\nonumber
    \end{align}
    and
    \begin{displaymath}
d[x_{p}^{(i)},x_{p'}^{(j)}]_t = \phi'(0)^2\big(\sigma_b^2 + \frac{\sigma_w^2}{D} \dprod{\psi(x^{(i)}_{t,\patch p})}{\psi(x^{(j)}_{t, \patch p'})}\big) \I_D dt
\end{displaymath}
\end{corollary}

In view of the results obtained in this section, all the remarks of \cref{sec:residual_diffusions_sub1} have a corresponding remark that applies to infinitely deep convolutional ResNets. Namely, the main qualitative conclusions continue to hold. That is, the stochastic process limit is well-behaved and perfect-correlation problems are avoided, explosive solutions are possible whenever $\phi''(0) \neq 0$.

\section{Infinitely deep ResNets' gradient}\label{sec:residual_diffusion_gradient}

We consider the problem of trainability at initialization of very deep ResNets which are finitely wide. In a generic setting, gradient descent iterations with a fixed learning rate $\eta$ are of the form
\begin{equation*}
    \theta(b+1) = \theta(b) -\eta \nabla \mathcal{R}(\theta(b)))
\end{equation*}
for $b=0,1,\dots$, where $\theta(b) \in \R^{\Theta}$ is the generic iteration of network's parameters of interest, and $\mathcal{R}(\theta)$ is a smooth real-valued loss function to be minimized. Differently from the gradient descent, the SGD relies on unbiased estimates of the gradient of the loss function of interest. In particular, for $\E[\nabla \mathcal{R}_b(\theta)] = \nabla \mathcal{R}(\theta)$, SGD iterations with a fixed learning rate $\eta$ are of the form
\begin{equation*}
    \theta(b+1) = \theta(b) -\eta \nabla \mathcal{R}_b(\theta(b)).
\end{equation*}
Both $\mathcal{R}(\theta)$ and $\mathcal{R}_b(\theta)$ are obtained by summing or averaging terms of the form $R(\widehat{y}(z), y)$ with $R: \R^Y \times \R^Y \rightarrow \R$ being the loss function for 1 data-point $(z,y)$ and $\widehat{y}(z)$ being the prediction of the neural network for $z$. For the rest of this section we consider a single data point and smooth $R$. A key difficulty in training very deep neural networks is that the gradients with respect to lower layers, i.e. small $t$ for large $L$ in our setting, might vanish or explode. This phenomenon results in negligible or diverging network's parameter updates and ultimately in bad training performance. This intuition can be made rigorous by linking the norm of the gradients, or their expectations, to loss function decrements \citep{bottou2018optimization}. 

For each $t$, let $\theta_t$ denote the weight, or the bias, at layer $t$, either in the ``standard'' form $(\Delta W_t, \Delta b_t)$ or in the ``reparametrized'' form $(\varepsilon^W_t, \varepsilon^b_t)$. Then, we can write the following equations
\begin{displaymath}
(\nabla_{\theta_{t-\Dt}}R)\tran = J(R,\xdt_T)J(\xdt_T,\xdt_t)J(\xdt_t,\theta_{t-\Dt})
\end{displaymath}
and
\begin{displaymath}
J(\xdt_T,\xdt_t) = J(\xdt_T,\xdt_{T-\Dt})J(\xdt_{T-\Dt},\xdt_{T-2\Dt}) \cdots J(\xdt_{t+\Dt},\xdt_t).
\end{displaymath}
The problematic term is represented by the Jacobian matrix $J(\xdt_T,\xdt_t)$. Indeed the matrix $J(\xdt_T,\xdt_t)$ involves a large (infinite in the limit $L \uparrow \infty$) number of matrix multiplications for the lower layers of a neural network, where $t \approx 0$. Note that $J(\xdt_T,\xdt_t)$ is closely related to $J(\xdt_t,\xdt_0)$ as, provided that $J(\xdt_t,\xdt_0)$ is invertible, $J(\xdt_T,\xdt_t)$ can be obtained as $J(\xdt_T,\xdt_t)=J(\xdt_T,\xdt_0)J(\xdt_t,\xdt_0)^{-1}$. In any case, the properties of $J(\xdt_t,\xdt_0)$ are most closely related to the problem of a vanishing/exploding gradient. Hereafter, we show that for infinitely deep ResNets, under suitable assumptions on the activation functions $\phi$ and $\psi$, the problem of an exploding gradient is avoided. Moreover, we show that under the same assumptions on $\phi$ and $\psi$ the limiting process is invertible. Construction of invertible neural networks is the main focus of recent research \citep{behrmann2019invertible}, and the invertibility of ResNets has been empirically shown to be related to model robustness \citep{logan2019adversarial}.

Let $\xdt_t$ follow the ResNet \cref{eq:resnet_fc} with the activation functions $\psi$ and $\phi$ satisfying \cref{ass:activation}. Let $\gdt_t = J(\xdt_t,\xdt_0)$, hence $\gdt_{t+\Dt} = J(\xdt_{t+\Dt},\xdt_t)\gdt_t$, and by direct computation we can write
\begin{displaymath}
    \Delta \gdt_t = \big(\phi'(\Delta W_t \psi(\xdt_t) + \Delta b_t) {1_D}\tran \odot \Delta W_t \odot 1_{D} \psi'(\xdt_t)\tran \big)\gdt_t.
\end{displaymath}
Now, we show that the Jacobian matrix $J(\xdt_t,\xdt_0)$ is well behaved in the sense that it converges to the solution $J(x_t,x_0)$ of a matrix SDE as $L \uparrow \infty$. As in the case of $\xdt_t$, we can derive a limiting SDE to which $\gdt_t$ converges, as $L \uparrow \infty$, by establishing the convergence of the corresponding instantaneous mean and covariance of $\gdt_t$. We denote this limiting SDE with $g_t = J(x_t,x_0)$. Subsequently we can link $g_t$ with $W_t$ and $b_t$ by showing the equivalence in law between $g_t$ and the solution to another SDE. The next theorem states directly the final result.

\begin{theorem} \label{thm:jacobian_resnet_fc}
Let $\gdt_t = J(\xdt_t,\xdt_0)$ and let $\overline{\gdt}_t$ denote the continuous-time interpolation of $\gdt_t $. Under \cref{ass:diffusion_parameters} and \cref{ass:activation}, $\overline{\gdt}_t$ converges in distribution to the solution of the matrix SDE
\begin{equation} \label{eq:jacobian_resnet_fc}
    dg_t = \left(\left(\phi'(0)dW_t + \phi''(0)d[W \psi(x) {1_D}\tran \odot W]_t\right) \odot 1_{D}\psi'(x_t)\tran\right) g_t
\end{equation}
\end{theorem}

We require the matrix SDE \cref{eq:jacobian_resnet_fc} to be non-explosive. This is for instance the case when $\phi''(0)$ and $\psi$ is the identity function. See also \cref{rem:explosion} in \cref{sec:residual_diffusions_sub1}. As long as \cref{eq:jacobian_resnet_fc} is not explosive, $J(x_t,x_0)$ is invertible and we can find the SDE determining the evolution of its inverse.

\begin{corollary} \label{thm:jacobian_resnet_fc_inverse}
Let $g_{t}$ be the matrix-valued diffusion process satisfying the matrix SDE \cref{eq:jacobian_resnet_fc}. Under \cref{ass:diffusion_parameters} and \cref{ass:activation}, $g_{t}$ is invertible and its inverse satisfies the matrix SDE
\begin{equation} \label{eq:jacobian_resnet_fc_inverse}
    dg_t^{-1} = g_t^{-1}\big(-\left(\phi'(0)dW_t + \phi''(0)d[W \psi(x) {1_D}\tran \odot W]_t\right) \odot 1_{D}\psi'(x_t)\tran + \phi'(0)^2 d[W \odot 1_{D}\psi'(x)\tran]_t \big).
\end{equation}
\end{corollary}

Hence, $J(x_T,x_t)$ can be obtained as $J(x_T,x_t) = g_T g_t^{-1}$ by means of integrating \cref{eq:jacobian_resnet_fc} and \cref{eq:jacobian_resnet_fc_inverse}, which are driven by the same process $W$. \cref{thm:jacobian_resnet_fc} and \cref{thm:jacobian_resnet_fc_inverse} have two fundamental consequences: i) as $g_t$ is the Jacobian of the last layer with respect of the first layer of the limiting process $x_t$, it follows that the exploding gradient problem is avoided as long as \cref{eq:jacobian_resnet_fc} is not explosive; ii) by the inverse function theorem the limiting process $x_t$ is invertible. Note that the results of this section hold for all the parametrizations discussed in \cref{sec:residual_diffusions_sub1}.

\section{Doubly infinite ResNets}\label{sec:doubly_infinite}

We study the more general setting of infinitely deep and infinitely wide ResNets, that is ResNets where both the depth $L$ and the dimension $D$ grow unboundedly. In particular, it is assumed that first $L \uparrow \infty$, and then $D \uparrow \infty$. Most of the analysis that follows assumes that $\psi$ is the identity function. Although we only provide heuristics for the convergence results of \cref{thm:layer_function_ode} and \cref{thm:doubly_infinite_sde}, numerical experiments reported in \cref{sec:numerical_experiments} support their correctness. Moreover, the numerical experiments of \cref{sec:numerical_experiments} support the conjecture that analogous results hold when $D$ and $L$ grow unbounded jointly when $\psi$ is the identity function, and $\phi$ is suitably smooth as assumed thorough this work. A more detailed discussion on interchanging the width and depth limits under these assumptions is contained in \cref{sec:note_double_infinite}. Hereafter, we consider the setting of \cref{thm:resnet_fc_iid}, i.e. fully i.i.d. network's parameters. In this case, we can rewrite \cref{eq:resnet_fc_iid} as 
\begin{equation}\label{eq:doubly_infinite_fc}
dx_t^{(i)} = \phi'(0)(\frac{\sigma_w}{\sqrt{D}}dB_t^W\psi(x_t^{(i)}) + \sigma_b dB_t^b) + \frac{1}{2}\phi''(0)(\sigma_b^2 + \sigma_w^2 q_t^{(i)}) 1_D dt
\end{equation}
and
\begin{displaymath}
d[x^{(i)}, x^{(j)}]_t = \phi'(0)^2 (\sigma_b^2 + \sigma_w^2 \lambda_t^{(i,j)}) \I_{D} dt \nonumber
\end{displaymath}
with $\lambda_t^{(i,j)}=\dprod{\psi(x_t^{(i)})}{\psi(x_t^{(j)})}/D$ and $q_t^{(i)}=\lambda_t^{(i,i)}=\norm{\psi(x_t^{(i)})}^2/D$. As a starting point, we need to ensure the well-posedness of \cref{eq:doubly_infinite_fc} for small $t>0$ as $D \uparrow \infty$. Therefore, we assume that the following limits exist and are finite: $q_{0}^{(i),\infty} = \lim_{D \uparrow \infty} q_0^{(i)}$, and $\lambda_{0}^{(i,j),\infty} = \lim_{D \uparrow \infty} \lambda_0^{(i,j)}$. Note that the notation does not convey explicitly the dependence of $x_t$, and hence of $q_t,\lambda_t$ on $D$.

\subsection{Weakly and fully trained ResNets}\label{sec:doubly_infinite_weak_full}

The connection between Gaussian processes and infinitely wide neural networks is well-known \citep{neal1995bayesian,lee2018deep,garriga-alonso2018deep}. In \cref{sec:doubly_infinite_diff} we show that similar results hold true for infinitely deep and infinitely wide ResNets, thus obtaining convergence to a Gaussian process. For infinitely wide neural networks, the NTK of \cite{jacot2018neural,arora2019exact,lee2019wide} allows for computing the solution obtained by fully training a neural network according to continuous-time, i.e. infinitesimal learning rate, gradient descent under the assumption of a quadratic loss. More in detail, let $\widehat{y}^{(i)}_\theta \in \R$ be the output of a neural network with parameters $\theta \in \R^{\Theta}$ for its $i$-th input, $i=1,\dots,N$. Let $\mathcal{R}(\theta) = \frac{1}{2}\sum_{i=1}^N (\widehat{y}_\theta^{(i)} - y^{(i)})^2$ be the squared loss over training data, where $y^{(i)}$ denotes the $i$-th target. We report \cite[Proposition 3.1]{arora2019exact} according to our notation:

\begin{proposition}
Consider minimizing the squared loss $\mathcal{R}(\theta)$ by gradient descent with infinitesimally small learning rate: $\frac{ d \theta (t)}{ d t}=-\nabla_\theta \mathcal{R}( \theta (t))$. Let $\widehat{y}(t)=\{\widehat{y}^{(i)}_{\theta(t)}\}_{i=1}^N \in \R^{N}$ be the network outputs on all inputs at time $t$, and $y = \{y^{(i)}\}_{i=1}^N$ be the desired outputs. Then $\widehat{y}(t)$ follows the following evolution, where $\mathcal{K}(t)$ is an $N \times N$ positive semidefinite matrix whose $(i,j)$-th entry is $\mathcal{K}(t)^{(i,j)}=\dprod{\nabla_{\theta}\widehat{y}_{\theta(t)}^{(i)}}{\nabla_{\theta}\widehat{y}_{\theta(t)}^{(j)}}$:
\begin{equation}\label{eq:arora_ntk}
\frac{d\widehat{y}(t)}{dt}=- \mathcal{K}(t)( \widehat{y}(t) - y ).
\end{equation}
\end{proposition}
According to Equation \eqref{eq:arora_ntk}, the empirical NTK corresponding to inputs $i,j$ is then defined as follows:
\begin{equation}\label{eq:ntk}
    \mathcal{K}^{(i,j)} = \dprod{\nabla_{\theta}\widehat{y}_{\theta(0)}^{(i)}}{\nabla_{\theta}\widehat{y}_{\theta(0)}^{(j)}},
\end{equation}
i.e. it is $\mathcal{K}(t)$ for $t=0$. As the width of a neural network goes to infinity, $\mathcal{K}^{(i,j)}$ converges to a deterministic limit $\mathcal{K}^{(i,j,\infty)}$ for each pair of points under the considered assumptions. Moreover, it is possible to bound the fluctuations of $\mathcal{K}(t)$ around its initial value $\mathcal{K}(0)$ as the width increases. Building on these results, it is possible to establish the equivalence between the solution obtained by fully training a neural network via continuous-time gradient descent and kernel regression via the $\mathcal{K}^{(i,j,\infty)}$ kernel. Indeed, within the aforementioned setting, the time-independent $\mathcal{K}^{(i,j,\infty)}$ can be substituted for $\mathcal{K}(t)^{(i,j)}$ in \cref{eq:arora_ntk} yielding the constant coefficients ODE
\begin{equation}\label{eq:arora_ntk_2}
\frac{d\widehat{y}(t)}{dt}=- \mathcal{K}^{(\infty)}( \widehat{y}(t) - y ).
\end{equation}
The stationary solution of the ODE \cref{eq:arora_ntk_2}, with initial condition $\widehat{y}(0) = 0$, is given by the standard kernel regression formula associated to the kernel $\mathcal{K}^{(\infty)}$. In particular, the desired initial condition can be enforced either by scaling the neural network output by an appropriately small constant \citep{arora2019exact} or by subtracting the output of an independent and untrained copy of the considered neural network at initialization \citep{lee2019wide}. 

We show in \cref{sec:doubly_infinite_diff_grad} that similar results hold for the case of infinitely deep and infinitely wide ResNets: \cref{eq:ntk} at initialization converges to a deterministic limit. Also, it is known \citep{arora2019exact} that in the aforementioned setting, training only the last output of a neural network under the same conditions corresponds to performing Bayesian inference under the Gaussian process prior arising in the infinite wide limit. Hence, we talk equivalently of Bayesian inference and weak training, and we refer to the standard NTK setting as full training. All the results of \cref{sec:doubly_infinite_diff} and \cref{sec:doubly_infinite_diff_grad} concern with a ResNet $\xdt_{0:T}$ with input $\xdt_0$. As previously mentioned, it is necessary to complete the ResNet with an input layer adapting the infinitely wide ResNet to finite-dimensional inputs. Moreover, to more closely resemble neural networks used in practice, an output adaptation layer is commonly introduced as well. In \cref{sec:doubly_infinite_diff_complete} we study the implications to the training of completing the ResNet with input and output layers. Moreover, in line with the NTK literature, it is necessary to consider an appropriate parametrization in order to obtained the desired NTK convergence results. The reparametrized gradients used by gradient descent are computed with respect to network's parameters which are i.i.d. distributed as standard Gaussian distributions and any scaling is expressed via multiplication, not via the Gaussian distribution's variance. In our context this corresponds to gradients with respect to $(\varepsilon^W_t, \varepsilon^b_t)$ in \cref{eq:fc_iid_discr_std1} and \eqref{eq:fc_iid_discr_std2}. 

\subsection{Weakly trained ResNets - general case}

Observe  that the evolution of \cref{eq:doubly_infinite_fc} is directly governed by $q_t^{(i)}$ and $\lambda_t^{(i,j)}$. The following results forms the basis of our analysis. Its proof is obtained by a straightforward but tedious application of multi-dimensional Ito's formula \citep[Section 4.2]{oksendal2003stochastic} and is thus omitted.

\begin{proposition}\label{thm:generic_q_l_evolution}
Let $x_t=x_t^{(i)}$ evolve according to the SDE \cref{eq:doubly_infinite_fc}, and let $f(x) = \dprod{g(x)}{h(x)}/D$ with $g$ and $h$ being twice continuously differentiable scalar functions applied element-wise. Then
\begin{align*}
df(x_t) &= \phi'(0)\frac{1}{D}(g'(x_t)h(x_t) + g(x_t)h'(x_t))\tran(\frac{\sigma_w}{\sqrt{D}}dB_t^W\psi(x_t) + \sigma_b dB_t^b)\\
&\quad+\frac{1}{2}(\sigma_b^2 + \sigma_w^2 q_t^{(i)})\phi'(0)^2(\dprod{g''(x_t)}{h(x_t)}/D + \dprod{g(x_t)}{h''(x_t)}/D + 2 \dprod{g'(x_t)}{h'(x_t)}/D)dt\\
&\quad+\frac{1}{2}(\sigma_b^2 + \sigma_w^2 q_t^{(i)})\phi''(0)(\dprod{g'(x_t)}{h(x_t)}/D + \dprod{g(x_t)}{h'(x_t)}/D)dt
\end{align*}
Moreover,  Let $x_t=x_t^{(i)}$ and $y_t=x_t^{(j)}$ evolve according to \cref{eq:doubly_infinite_fc}, and let $F(x,y) = \dprod{G(x)}{H(y)}/D$ with $G$ and $H$ being twice continuously differentiable scalar functions applied element-wise. Then
\begin{align*}
dF(x_t,y_t) &= \phi'(0)\frac{1}{D}(G'(x_t)H(y_t))\tran(\frac{\sigma_w}{\sqrt{D}}dB_t^W\psi(x_t) + \sigma_b dB_t^b)\\
&\quad+ \phi'(0)\frac{1}{D}(G(x_t)H'(y_t))\tran(\frac{\sigma_w}{\sqrt{D}}dB_t^W\psi(y_t) + \sigma_b dB_t^b)\\
&\quad+\frac{1}{2D}(\sigma_b^2 + \sigma_w^2 q_t^{(i)})\phi''(0)\dprod{G'(x_t)}{H(y_t)} dt\\
&\quad+\frac{1}{2D}(\sigma_b^2 + \sigma_w^2 q_t^{(j)})\phi''(0)\dprod{G(x_t)}{H'(y_t)} dt\\
&\quad+\frac{1}{2D}(\sigma_b^2 + \sigma_w^2 q_t^{(i)})\dprod{G''(x_t)}{H(y_t)} dt\\
&\quad+\frac{1}{2D}(\sigma_b^2 + \sigma_w^2 q_t^{(j)})\dprod{G(x_t)}{H''(y_t)} dt\\
&\quad+\frac{1}{D}(\sigma_b^2 + \sigma_w^2 \lambda_t^{(i,j)})\dprod{G'(x_t)}{H'(y_t)} dt
\end{align*}
\end{proposition}

The evolution of $q_t^{(i)}$ and $\lambda_t^{(i,j)}$ is obtained by an application of \cref{thm:generic_q_l_evolution} with $h \coloneqq g \coloneqq H \coloneqq G \coloneqq \psi$, resulting in $dq_t^{(i)}=df(x_t^{(i)})$ and $d\lambda_t^{(i,j)}=dF(x_t^{(i)},x_t^{(j)})$. Inspecting the corresponding dynamics highlights the difficulties in obtaining closed-form solutions for the evolution of $q_t^{(i)}$ and $\lambda_t^{(i,j)}$. The drift and diffusion coefficients describing the evolution of $q_t^{(i)}$ involves inner-products of the form $\dprod{\psi^{(m)}(x_t^{(i)})}{\psi^{(n)}(x_t^{(i)})}$ where $m,n=0,1,2$ denote the degrees of differentiation. By relying on \cref{thm:generic_q_l_evolution} to determine the evolution of $\dprod{\psi^{(m)}(x_t^{(i)})}{\psi^{(n)}(x_t^{(i)})}$ results in drift and diffusion coefficients depending on other terms of the form $\dprod{\psi^{(m)}(x_t^{(i)})}{\psi^{(n)}(x_t^{(i)})}$ involving higher degrees of differentiation. Similar considerations hold true for $\lambda_t^{(i,j)}$. Ultimately, determining the evolution of $q_t^{(i)}$ and $\lambda_t^{(i,j)}$ separately from $x_t^{(i)}$ requires solving an infinite recursion of SDEs. Similarly, solving \cref{eq:doubly_infinite_fc} directly requires solving an infinite dimensional SDE as $D \uparrow \infty$. In either cases it doesn't seem possible to solve \cref{eq:doubly_infinite_fc} exactly in the infinite-width limit with a finite amount of computational effort. In the special case where $\psi$ is the identity function the dynamics of $q_t^{(i)}$ and $\lambda_t^{(i,j)}$ simplify considerably. We explore the implications of this modeling assumption in the remaining sections of this work.

\subsection{Weakly trained ResNets - identity $\psi$ case}\label{sec:doubly_infinite_diff}

Hereafter, we assume the activation function $\psi$ to be the identity function, and hence only the activation function $\phi$ affects the neural network. In \cref{thm:sde_layers} in \cref{app:proofs} we derive the evolution of $q_t^{(i)}$ and $\lambda_t^{(i,j)}$ in this specific setting. Now, we show that as $D$ increases $q_t^{(i)}$ and $\lambda_t^{(i,j)}$ converge to deterministic limits which are obtained as solutions of a finite-dimensional ODE system.

\begin{proposition}[Heuristic] \label{thm:layer_function_ode}
As $D \uparrow \infty$, the quantities $m_t^{(i)}$, $q_t^{(i)}$, $\lambda_t^{(i,j)}$ converge to the solutions of the ODEs
\begin{displaymath}
dm_t^{(i),\infty} = \frac{1}{2} \phi''(0) \big(\sigma_b^2 + \sigma_w^2 q_t^{(i),\infty}\big) dt,
\end{displaymath}
\begin{displaymath}
dq_t^{(i),\infty} = \big(\phi''(0)m_t^{(i),\infty} + \phi'(0)^2\big) \big(\sigma_b^2 + \sigma_w^2 q_t^{(i),\infty}\big)dt
\end{displaymath}
and
\begin{align*}
&d\lambda_t^{(i,j),\infty} = \big(\frac{1}{2}\phi''(0)((\sigma_b^2 + \sigma_w^2 q_t^{(i),\infty})m_t^{j,\infty} + (\sigma_b^2 + \sigma_w^2 q_t^{(j),\infty})m_t^{(i),\infty})\\
    &\quad\quad\quad\quad+ \phi'(0)^2\big(\sigma_b^2 + \sigma_w^2 \lambda_t^{(i,j),\infty})\big)dt,
\end{align*}
respectively. Moreover, under the assumption $\phi''(0)=0$, the solutions for $m_T^{(i),\infty}$, $q_T^{(i),\infty}$ and $\lambda_T^{(i,j),\infty}$ are
\begin{displaymath}
m_T^{(i),\infty} = m_0^{(i),\infty},
\end{displaymath}
\begin{displaymath}
q_T^{(i),\infty} = q_0^{(i),\infty} + \left(q_0^{(i),\infty} + \frac{\sigma_b^2}{\sigma_w^2}\right)\left( e^{\phi'(0)^2\sigma_w^2 T}  - 1\right).
\end{displaymath}
and
\begin{displaymath}
\lambda_T^{(i,j),\infty} = \lambda_0^{(i,j),\infty} + \left(\lambda_0^{(i,j),\infty} + \frac{\sigma_b^2}{\sigma_w^2}\right)\left( e^{\phi'(0)^2\sigma_w^2 T}  - 1\right),
\end{displaymath}
respectively. Moreover, under the assumption $\phi''(0) \neq 0$, the solutions for $m_T^{(i),\infty}$ and $q_T^{(i),\infty}$ are
\begin{equation}\label{eq:explosive_m}
m_T^{(i),\infty} = \frac{1}{\phi^{\prime\prime}(0)}\left\{-\phi'(0)^{2}+\frac{1}{\sigma_{w}}\sqrt{C}\tan\left(\frac{1}{2}\sigma_{w}\sqrt{C}(T+2c_{2})\right)\right\}
\end{equation}
and
\begin{equation}\label{eq:explosive_q}
q_T^{(i),\infty} = \frac{1}{\phi^{\prime\prime}(0)^{2}\sigma^{2}_{w}}\left\{-\phi^{\prime\prime}(0)^{2}\sigma^{2}_{b}+C\sec\left(\frac{1}{2}\sigma_{w}\sqrt{C}(T+2c_{2})\right)^{2}\right\},
\end{equation}
respectively, where $c_{1}$ and $c_{2}$ are two constants that depend on the initial conditions of \cref{eq:explosive_m,eq:explosive_q}, and $C=-\phi'(0)^{4}\sigma^{2}_{w}+\phi^{\prime\prime}(0)^{2}(\sigma^{2}_{b}+\sigma^{2}_{w}c_{1})$.
\end{proposition}

In \cref{thm:layer_function_ode}, $m_t^{(i)} = D^{-1}\sum_{1\leq d\leq D} x_{t,d}^{(i)}$ and we further require that $m_{0}^{(i),\infty} = \lim_{D \uparrow \infty} m_0^{(i)}$ exists and it is finite. Crucially, when $\psi$ is the identity function, it is not possible to decouple the evolution of $q_t^{(i)}$ and $\lambda_t^{(i,j)}$ from the evolution of the driving stochastic process $x_t$ when $D \uparrow \infty$. The drift, diffusion, and correlation coefficients driving the SDE \cref{eq:doubly_infinite_fc} converge to deterministic limit too which results in i.i.d. stochastic processes across the dimensions of the neural network.

\begin{proposition}[Heuristic] \label{thm:doubly_infinite_sde}
As $D \uparrow \infty$ each $x_t^{(i)}$ converges to $x_t^{(i),\infty}$, with the $x_t^{(i),\infty}$'s being i.i.d. across the dimensions of the neural network. Moreover, for $x_t^{(i),\infty} = x_{t,1}^{(i),\infty},x_{t,2}^{(i),\infty},\dots$, and $d,u \geq 1$ it holds
\begin{equation}\label{eq:doubly_infinite_fc_ode}
dx_{t,d}^{(i),\infty} = \phi'(0) (\sigma_b^2 + \sigma_w^2 q_t^{(i),\infty})^{1/2}dB^{(i),\infty}_{t,d}
\end{equation}
and
\begin{displaymath}
d[B^{(i),\infty}_d,B^{(j),\infty}_u]_t = \frac{\sigma_b^2 + \sigma_w^2 \lambda_t^{(i,j),\infty}}{\big((\sigma_b^2 + \sigma_w^2 q_t^{(i),\infty})(\sigma_b^2 + \sigma_w^2 q_t^{(j),\infty})\big)^{1/2}} \delta_{d,u} dt,
\end{displaymath}
where $B^{(i),\infty}_{t,1},B^{(i),\infty}_{t,2},\dots$ are scalar BMs dependent over $i$ and $q_t^{(i),\infty},\lambda^{(i,j),\infty}_t$ are obtained by solving the ODEs in \cref{thm:layer_function_ode}. Over the two data-points indexed by $i,j$ this is a 2-dimensional SDE with time-dependent and deterministic drift and diffusion coefficients, and such that
\begin{align} \label{eq:fc_iid_infty_analytical}
    p(x_{T,d}^{(i),\infty},x_{T,d}^{(j),\infty}\,|\,x_{0,d}^{(i)},x_{0,d}^{(j)}) = \mathcal{N}_2\Bigg(&\begin{bmatrix*} 
        x_{0,d}^{(i)} + m_T^{(i),\infty} - m_0^{(i),\infty}\\
        x_{0,d}^{(j)} + m_T^{(j),\infty} - m_0^{(j),\infty}
    \end{bmatrix*},\\
    &\begin{bmatrix*} 
        v_T^{(i),\infty} - v_0^{(i),\infty} & c_T^{(i,j),\infty} - c_0^{(i,j),\infty}\\
        c_T^{(i,j),\infty} - c_0^{(i,j),\infty} & v_T^{(j),\infty} - v_0^{(j),\infty}
    \end{bmatrix*}\Bigg),\nonumber
\end{align}
where $v_t^{(i),\infty} = q_t^{(i),\infty} - (m_t^{(i),\infty})^2$, $c_t^{(i,j),\infty} = \lambda_t^{(i,j),\infty} - m_t^{(i),\infty}m_t^{(j),\infty}$.
\end{proposition}

According to \cref{thm:doubly_infinite_sde}, doubly infinite ResNets are non-centered Gaussian processes with covariance kernel $K^{(i,j),\infty} = c_T^{(i,j),\infty} - c_0^{(i,j),\infty}$ and mean function $M_d^{(i),\infty} = x_{0,d}^{(i)} + m_T^{(i),\infty} - m_0^{(i),\infty}$. That is: i)  when $\phi''(0) = 0$, we have $M_d^{(i),\infty} = x_{0,d}^{(i)}$ and $K^{(i,j),\infty} =(\lambda_0^{(i,j),\infty} + \sigma_b^2/\sigma_w^2)( e^{\phi'(0)^2\sigma_w^2 T}  - 1)$; ii) when $\phi''(0) \neq 0$, from Equation \cref{eq:explosive_m} and Equation \cref{eq:explosive_q} we obtain the deterministic explosion time of $x_{t,d}^{(i),\infty}$ by solving $2^{-1}\sigma_{w}\sqrt{C}(T+2c_{2}) = \pi/2$ in $T$; in particular, the constants $c_1$, $c_2$ depend on $m_0^{(i),\infty}$, $q_0^{(i),\infty}$ and have to be determined numerically.

\subsection{Fully trained ResNets}\label{sec:doubly_infinite_diff_grad}

Let $\theta = \{\varepsilon^W_t,\varepsilon^b_t\}_{t=0}^{T-1}$ denote the ``reparametrized'' collection of network's parameters with respect to which we compute the NTK. We establish the convergence of $\mathcal{K}^{(i,j)}$ to a deterministic limit as $L \uparrow \infty$ and then $D \uparrow \infty$. We operate under the following assumptions: i) $\xdt_t$ follows \cref{eq:resnet_fc_iid} with $\psi$ being the identity function; ii) network's parameters follow \cref{ass:diffusion_parameters_iid}; iii) $\phi''(0) = 0$; iv) $\widehat{y} = \xdt_{T,1}$. We have $\mathcal{K}^{(i,j)} = \mathcal{K}_W^{(i,j)} + \mathcal{K}_b^{(i,j)}$, where $\mathcal{K}_W^{(i,j)} = \sum_{t=\Dt}^T \mathcal{K}_{W,t}^{(i,j)}$, $\mathcal{K}_b^{(i,j)} = \sum_{t=\Dt}^T \mathcal{K}_{b,t}^{(i,j)}$, and
\begin{align*}
    &\mathcal{K}_{W,t}^{(i,j)} =J(\widehat{y}^{(i)},\xdt_T^{(i)})J(\xdt_T^{(i)},\xdt_t^{(i)})J(\xdt_t^{(i)},\Delta W_{t-\Dt})\\
    &\quad\quad\quad\quad\times\Big(J(\widehat{y}^{(j)},\xdt_T^{(j)})J(\xdt_T^{(j)},\xdt_t^{(j)})J(\xdt_t^{(j)},\Delta W_{t-\Dt})\Big)\tran {\sigma_w^2} \Dt / D
    \end{align*}
and
    \begin{align*}
    &\mathcal{K}_{b,t}^{(i,j)} =J(\widehat{y}^{(i)},\xdt_T^{(i)})J(\xdt_T^{(i)},\xdt_t^{(i)})J(\xdt_t^{(i)},\Delta b_{t-\Dt})\\
    &\quad\quad\quad\quad\times\Big(J(\widehat{y}^{(j)},\xdt_T^{(j)})J(\xdt_T^{(j)},\xdt_t^{(j)})J(\xdt_t^{(j)},\Delta b_{t-\Dt})\Big)\tran {\sigma_b^2} \Dt,
\end{align*}
as
\begin{displaymath}
J(\xdt_t,\zeta^W_{t-\Dt}) = J(\xdt_t,\Delta W_{t-\Dt}) {\sigma_w} \sqrt{\Dt} / \sqrt{D}
\end{displaymath}
and
\begin{displaymath}
J(\xdt_t,\zeta^b_{t-\Dt}) = J(\xdt_t,\Delta b_{t-\Dt}) {\sigma_b} \sqrt{\Dt}.
\end{displaymath}

Recall from our study in \cref{sec:residual_diffusion_gradient} that, as $L \rightarrow \infty$, we have that $J(\xdt_T^{(i)},\xdt_t^{(i)}) \rightarrow g_Tg_t^{-1}$ and $J(\xdt_T^{(j)},\xdt_t^{(j)}) \rightarrow g_Tg_t^{-1}$, as the evolution of $g_t$ does not depend on $x_t$ when $\phi''(0) = 0$. Furthermore, note that $J(\xdt_t,\Delta W_{t-\Dt})_{d,i,j} \rightarrow \phi'(0) \delta_{d,i} x_{t,j}$ and $J(\xdt_t,\Delta b_{t-\Dt})_{d,i} \rightarrow \phi'(0) \delta_{d,i}$. By combining these results, and by assuming that the interchange of limits is justified, we write
\begin{displaymath}
\mathcal{K}_W^{(i,j)} \rightarrow \phi'(0)^2 {\sigma_w^2} J(\widehat{y}^{(i)},x_T^{(i)}) g_T \left[\int_0^T \frac{\dprod{x_t^{(i)}}{x_t^{(j)}}}{D} g_t^{-1} {g_t^{-1}}\tran dt \right] g_T\tran J(\widehat{y}^{(j)},x_T^{(j)})\tran
\end{displaymath}
and
\begin{displaymath}
\mathcal{K}_b^{(i,j)} \rightarrow \phi'(0)^2 {\sigma_b^2} J(\widehat{y}^{(i)},x_T^{(i)}) g_T \left[\int_0^T g_t^{-1} {g_t^{-1}}\tran dt \right] g_T\tran J(\widehat{y}^{(j)},x_T^{(j)})\tran.
\end{displaymath}
Now, $g_t^{-1} {g_t^{-1}}\tran = (g_t\tran g_t)^{-1}$. Accordingly, by an application of Ito's formula for matrix SDE products \citep[Chapter V, Theorem 47]{protter2005stochastic}, for $U_t = g_t\tran g_t$ we obtain the following SDE
\begin{equation*}
    dU_t = \phi'(0)\frac{\sigma_w}{\sqrt{D}}g_t\tran\left(dB_t^W + {dB^W_t}\tran\right)g_t + \phi'(0)^2\sigma_w^2 U_t dt,
\end{equation*}
where $U_0 = \I_D$, and whose quadratic variation (a matrix, in this particular case) is of the following form
\begin{equation*}
    d[U]_t = \phi'(0)^2\frac{\sigma_w^2}{D} \Big({g_t}\tran \odot {g_t}\tran\Big)\Big(g_t \odot g_t\Big)dt,
\end{equation*}
vanishing as $D \rightarrow \infty$. Therefore, $U_t \rightarrow U_t^{\infty}$ where $dU_t^{\infty} = \phi'(0)^2\sigma_w^2 U_t^{\infty} dt$. Thus, as $D \rightarrow \infty$ the term $g_t\tran g_t$ is an infinite dimensional diagonal matrix with constant element $u_t^{\infty}$ computable by solving the ODE $du_t^{\infty} = \phi'(0)^2\sigma_w^2 u_t^{\infty}dt$ with initial value $u_0^{\infty} = 1$, i.e. $u_t^{\infty} = \exp(\phi'(0)^2\sigma_w^2 t)$. 

Note that: i) the matrix $\int_0^T (g_t\tran g_t)^{-1}dt$ is asymptotically diagonal with constant element $(1 - \exp(\phi'(0)^2\sigma_w^2 T))/(\phi'(0)^2\sigma_w^2)$; ii) the matrix $g_T g_T\tran$ is asymptotically diagonal with constant element $\exp(\phi'(0)^2\sigma_w^2 T)/(\phi'(0)^2\sigma_w^2)$. Therefore, one has that the matrix $g_T \left[\int_0^T g_t^{-1} {g_t^{-1}}\tran dt \right] g_T\tran$ is asymptotically diagonal with value $(\exp(\phi'(0)^2\sigma_w^2 T) - 1)/(\phi'(0)^2\sigma_w^2)$. Note that we rely on the assumption that the approximation errors due to considering each term separately vanish as $D \uparrow \infty$, or at least the approximation errors cancel out. Finally, $\widehat{y} = \xdt_{T,1}$ corresponds to selecting the first element of this diagonal matrix. If  $E = e^{\phi'(0)^2\sigma_w^2 T}$ then
\begin{displaymath}
    \mathcal{K}_b^{(i,j)} \rightarrow \mathcal{K}_b^{(i,j),\infty} = {\frac{\sigma_b^2}{\sigma_w^2}}(E - 1).
\end{displaymath}
Along similar lines we obtain the deterministic limit to which $K_{\mathcal{NT},W,t}^{(i,j)}$ converges as $D \uparrow \infty$, i.e., 
\begin{displaymath}
    \mathcal{K}_W^{(i,j)} \rightarrow \mathcal{K}_W^{(i,j),\infty} = \lambda_0^{(i,j),\infty} \phi'(0)^2 {\sigma_w^2} T E + \frac{\sigma_b^2}{\sigma_w^2}\left[ \phi'(0)^2  {\sigma_w^2} T E - {(E - 1)} \right]
\end{displaymath}
hence obtaining
\begin{equation*}
    \mathcal{K}^{(i,j)} \rightarrow \mathcal{K}^{(i,j),\infty} = \lambda_0^{(i,j),\infty} {C}E + {\frac{\sigma_b^2}{\sigma_w^2} CE}.
\end{equation*}
where $E = \exp(C)$ and $C=\phi'(0)^2\sigma_w^2T$. This can be contrasted with the main result of \cref{sec:doubly_infinite_diff}, where we have shown that the (standard) kernel corresponding to $D \uparrow \infty$ for $\phi''(0) = 0$ is given by
\begin{equation*}
    K^{(i,j)} \rightarrow K^{(i,j),\infty} = \lambda_0^{(i,j),\infty}(E - 1) + \frac{\sigma_b^2}{\sigma_w^2}( E  - 1).
\end{equation*}
Note that the two kernels correspond to two different training regimes: i) training all layers of the neural network; ii) training only the output layer of the neural network. However, the two kernels are  qualitatively similar. In particular, both kernels depend linearly on $\lambda_0^{(i,j),\infty}$. The only difference is with respect to the behavior of $(E - 1)$ compared to $CE$ as a function of $C$.

\subsection{Training of completed ResNets} \label{sec:doubly_infinite_diff_complete}

Results presented in \cref{sec:doubly_infinite_diff} and \cref{sec:doubly_infinite_diff_grad} entail a neural network with an infinite-dimensional input. Let $z,z' \in \R^Z$\footnote{for convenience we use in this section the $z,z'$ notation instead of $z^{(i)},z^{(j)}$, and proceed in the same way for all other quantities depending on $i,j$} be two inputs of the neural network. We consider a linear adaptation layer, i.e. an embedding, of the form $\xdt_0 = A z$ where, in line with \cref{sec:doubly_infinite_diff} and \cref{sec:doubly_infinite_diff_grad}, the elements of $A \in \R^{D \times Z}$ are i.i.d. as $\mathcal{N}(0,\sigma_Z^2)$. It follows that across $d$ we have
\begin{displaymath}
(\xdt_{0,d},\xdt'_{0,d}) \overset{i.i.d.}{\sim} \mathcal{N}_2(0,\Sigma^Z(z,z')),
\end{displaymath}
where
\begin{equation*}
\Sigma^Z(z,z') = \sigma_Z^2 \begin{bmatrix*} \norm{z}^2 & \dprod{z}{z'} \\ \dprod{z'}{z} & \norm{z'}^2 \end{bmatrix*}.
\end{equation*}
By the strong law of large numbers $\lambda_0 = \frac{1}{D} \dprod{\xdt_0}{\xdt'_0} \rightarrow  \lambda_0^{\infty} = \E[\xdt_{0,1}\xdt'_{0,1}] = \sigma_Z^2 \dprod{z}{z'}$ as $D \uparrow \infty$, hence $q_0^{\infty} = \sigma_Z^2 \norm{z}^2$ and ${q'}_0^{\infty} = \sigma_Z^2 \norm{z'}^2$. In the weakly training setting, which is equivalent to Bayesian inference with a Gaussian process prior, we know from \cref{sec:doubly_infinite_diff} that across $d$ we have
\begin{equation*}
    (x_{T,d}^{\infty},{x'}_{T,d}^{\infty}\,|\,\xdt_{0,d},{\xdt'}_{0,d}) \overset{i.i.d.}{\sim} \mathcal{N}_2\Bigg(\begin{bmatrix*} 
        \xdt_{0,d}\\
        {\xdt'}_{0,d}
    \end{bmatrix*},
    \Sigma^{\text{weak}}(z,z')\Bigg),
\end{equation*}
where 
\begin{equation*}
\Sigma^{\text{weak}}(z,z') = \sigma_Z^2 (E-1) \begin{bmatrix*} \norm{z}^2 & \dprod{z}{z'} \\ \dprod{z'}{z} & \norm{z'}^2 \end{bmatrix*} + \frac{\sigma_b^2}{\sigma_w^2}(E - 1)
\end{equation*}
with $E = e^{\phi'(0)^2\sigma_w^2 T}$. Then, by direct computation, we obtain the following Gaussian distribution
\begin{align*}
    (x_{T,d}^{\infty},{x'}_{T,d}^{\infty}) \overset{i.i.d.}{\sim} &\mathcal{N}_2\Bigg(\begin{bmatrix*} 
        0\\
        0
    \end{bmatrix*},
    \Sigma^Z(z,z') + \Sigma^{\text{weak}}(z,z')\Bigg)\\
    =\ &\mathcal{N}_2\Bigg(\begin{bmatrix*} 
        0\\
        0
    \end{bmatrix*},
    \sigma_Z^2 E \begin{bmatrix*} \norm{z}^2 & \dprod{z}{z'} \\ \dprod{z'}{z} & \norm{z'}^2 \end{bmatrix*} + \frac{\sigma_b^2}{\sigma_w^2}(E - 1)\Bigg).
\end{align*}
That is, the prior distribution induced by a doubly infinite ResNet with the input adaptation layer is i.i.d. across the dimensions $d$, and distributed as a centered Gaussian process with kernel
\begin{equation}\label{eq:doubly_infinite_cgk}
\overline{K}(z,z') = \sigma_Z^2 E \dprod{z}{z'} + \frac{\sigma_b^2}{\sigma_w^2}(E - 1).
\end{equation}

We also augment the neural network with an output adaptation layer $\widehat{y} = G \xdt_T$, where the elements of $G \in \R^{1 \times D}$ are i.i.d. as $\mathcal{N}(0,\sigma_Y^2/D)$. Then, it follows that the doubly infinite ResNet with both input and output adaption layers still follows a Gaussian process whose kernel is
\begin{equation}\label{eq:doubly_infinite_cgk_2}
\overline{\overline{K}}(z,z') = \sigma_Y^2\overline{K}(z,z').
\end{equation}
Now, consider the Bayesian noiseless linear model with fully independent prior distributions formulated by $\widehat{y} = \alpha + \beta z$ where $\alpha \in \R$, $\alpha \sim \mathcal{N}(0, \sigma_{\alpha}^2)$, $\beta \in \R^Z$, $\beta_i \sim \mathcal{N}(0, \sigma_{\beta}^2)$ for $i=1,\dots,Z$, then:
\begin{equation*}
    (\widehat{y},\widehat{y}') \sim \mathcal{N}_2\Bigg(\begin{bmatrix*} 
        0\\
        0
    \end{bmatrix*},
\sigma_{\beta}^2 \begin{bmatrix*} \norm{z}^2 & \dprod{z}{z'} \\ \dprod{z'}{z} & \norm{z'}^2 \end{bmatrix*} + \sigma_{\alpha}^2 \Bigg).
\end{equation*}
Thus, according to \eqref{eq:doubly_infinite_cgk} it follows that, within the doubly infinite limit, the completed ResNet prior model collapses to a noiseless Bayesian linear regression prior where $\sigma_{\alpha}^2 = \frac{\sigma_b^2}{\sigma_w^2}\sigma_Y^2(E - 1)$ and $\sigma_{\beta}^2 = \sigma_Y^2 \sigma_Z^2 E$. 

Under the fully trained setting, in \cref{sec:doubly_infinite_diff_grad} we have established the convergence of the NTK. Now, we consider directly the doubly infinite ResNet augmented with both input and output layers as previously defined. Recall that in the NTK literature the input layer is sometimes not trained, and the output layer is sometimes omitted \citep{arora2019exact}. Hereafter, we report only the results for the special case in which all layers are present and trained as it most closely resembles standard practice for finitely-sized networks. In particular,
\begin{align}
\overline{\overline{\mathcal{K}}}(z,z') &= \sigma_Y^2\left(\sigma_Z^2 (C + 1)E \dprod{z}{z'} + \frac{\sigma_b^2}{\sigma_w^2} CE\right) + \overline{\overline{K}}(z,z')\notag\\
&=\sigma_Y^2\left(\sigma_Z^2 (C + 2)E \dprod{z}{z'} + \frac{\sigma_b^2}{\sigma_w^2} (CE +  E - 1)\right).\label{eq:doubly_infinite_ntk_2}
\end{align}
For more general cases, the results follow along lines similar to the steps detailed in \cref{sec:doubly_infinite_diff_grad}. In particular, in all cases the kernel remains affine in $\dprod{z}{z'}$, and only the coefficients are affected.

The problem of establishing the equivalence between kernel regression and fully trained neural networks requires two steps: i) establishing the NTK convergence at initialization, as we have proved in \cref{sec:doubly_infinite_diff_grad}; ii) bounding the NTK fluctuations during training \citep{arora2019exact,jaehoon2019wide}. To the best of our knowledge, step ii) has not been formally established for architectures which are not feed-forward, e.g. for ResNets. Assuming such a result, Equation  \cref{eq:doubly_infinite_ntk_2} shows that in the doubly infinite limit fully trained ResNet correspond to noiseless (kernel) linear regression. Kernel regression is equivalent to the posterior predictive mean of a Gaussian process with the same kernel. Hence, relatively to point predictions, both weakly and fully trained doubly infinite ResNets of the considered class collapse to a linear model.

\subsection{Concluding remarks} \label{sec:note_double_infinite}

We considered doubly infinite ResNets by first establishing a diffusion limit ($L \uparrow \infty$) and then by considering increasing dimensionality of the diffusion process ($D \uparrow \infty$). Here, we briefly present the alternative approach where we invert the order to taking limits, that is first the width tends to infinity, then the depth tends to infinity. We assume fully i.i.d. network's parameters according to a time-discretization of \cref{ass:diffusion_parameters_iid}, i.e. \cref{eq:fc_iid_discr_std1} and \cref{eq:fc_iid_discr_std2}, with $\sigma_b^2=0,\sigma_w^2=1$ for simplicity of exposition. We consider two inputs $x,y\in\R^I$ and the same input adaption layer of \cref{sec:doubly_infinite_diff_complete}, with $\sigma_I^2=1$. Once again for simplicity of exposition, we consider in this section the ResNet \cref{eq:resnet_fc} without the $\phi$ activation function. In such a setting, the case $L \uparrow \infty$ is thoroughly studied, starting with the seminal work of \cite{neal1995bayesian} and proceeding with the more recent developments of \cite{lee2018deep}, \cite{jacot2018neural} and \cite{lee2019wide}. By means of the standard approach we obtain the following recursion over the layers of an infinitely wide ResNet:
\begin{displaymath}
(x_{0,d}, y_{0,d}) \overset{i.i.d.}{\sim} \mathcal{N}_2(0, \Sigma_0)
\end{displaymath}
with
\begin{displaymath}
\Sigma_0^{xy}=\dprod{x}{y},
\end{displaymath}
and
\begin{displaymath}
(x_{t+\Dt,d}, y_{t+\Dt,d}) \overset{i.i.d.}{\sim} \mathcal{N}_2(0, \Sigma_{t+\Dt} = \Sigma_t + \Sigma_{t:t+\Dt}) 
\end{displaymath}
with
\begin{displaymath}
\Sigma_{t:t+\Dt}^{xy} = \E_{\varepsilon,\gamma \sim \Sigma_t^{xy}}[\psi(\varepsilon)\psi(\gamma)]\Dt,
\end{displaymath}
where $d=1,\dots$ are the neural network units, which are i.i.d. in the infinitely wide limit, and we used the notation $\Sigma^{xy}$ to denote upper-right element of a $2 \times 2$ covariance matrix (the remaining elements are obtained setting $x=y$ or $y=x$). The recursion over layer depth (time):
\begin{equation*}
\Sigma_{t+\Dt}^{xy} = \Sigma_t^{xy} + \E_{\varepsilon,\gamma \sim \Sigma_t^{xy}}[\psi(\varepsilon)\psi(\gamma)]\Dt
\end{equation*}
provides an easy way to establish the infinitely deep limit ($D \uparrow \infty$) via the ODE limit of the recursion over
\begin{equation}\label{eq:cov_ode}
\dot{\Sigma}_{t}^{xy} = \E_{\varepsilon,\gamma \sim \Sigma_t^{xy}}[\psi(\varepsilon)\psi(\gamma)].
\end{equation}
The extension to an arbitrary number of inputs is immediate. In particular, the resulting continuum of Gaussian processes, one at each $t \geq 0$, requires to solve the ODE \cref{eq:cov_ode}. The expectation defining the drift of \cref{eq:cov_ode} has a closed-form solution for some specific choices of the activation functions, and requires numerical approximations in the general case. The ODE  \cref{eq:cov_ode} also requires numerical integration aside from specific cases, such as $\psi$ being the identity function. It is easy to check that in this case one obtains the same results of \cref{sec:doubly_infinite_diff} for $\phi'(0)=1$ (for instance when $\phi=\tanh$). Thus, in this setting, the order in which the limits of depth and width are taken does not matter. When $\phi'(0) \neq 1$, the results of this Section only differ by the multiplicative constant $\phi'(0)$ which can be absorbed into $\sigma_b$ and $\sigma_w$, resulting again in an equivalent model.  Moreover, the derivations of this Section can be extended to the cover the case of a suitably smooth activation $\phi$, in which case the constant $\phi'(0)$ would be recovered as well. Ultimately, the $\phi$ activation plays a very limited role for large $L$, and when the $\psi$ activation is missing (i.e. it is the identity) the ResNet tends to a linear neural network in the doubly infinite limit.

\section{Numerical results}\label{sec:numerical_experiments}

We start by introducing all the neural network models considered in this section. In all the experiments we set $\psi$ to the identity function and, without loss of generality, we assume $T=1$.

Regarding fully-connected networks, we consider the fully i.i.d. parametrization of \cref{ass:diffusion_parameters_iid}. When $Z=1$, i.e. for 1-dimensional inputs, we can opt for copying the input across all dimensions: $\xdt_{0,\bullet} = z$ for an input $z$, i.e. $\xdt_{0,d} = z$ for each $d \in D$. We refer to this model as $\mathcal{F}_{\tanh}$ when $\phi=\tanh$ and as $\mathcal{F}_{\swish}$ when $\phi=\swish$. The $\swish$ activation function ($\swish(x) = x \sigmoid(x)$) has been shown empirically \citep{ramachandran2017searching} and theoretically \citep{hayou2019impact} to be competitive. More in general, for any input dimension $Z$, we complete the model with input and output adaptation layers as defined in \cref{sec:doubly_infinite_diff_complete}. We choose to use $\sigma_Z^2=Z/I$ and $\sigma_Y=1/D$. We will refer to such completed models as $\overline{\overline{\mathcal{F}}}_{\tanh}$ and $\overline{\overline{\mathcal{F}}}_{\swish}$. 

Regarding convolutional networks, we consider the fully i.i.d. parameterization of \cref{ass:diffusion_parameters_iid_cnn}. A generic input is here of dimension $U \times V \times C$, with $U,V,C$ representing the input height, width and number of channels. The adaptation layer is here an 1-by-1 convolution adapting the number of channels to the model dimension $D$. More precisely: for each $p$ $\xdt_{0,p} = A z_{p}$ where $p$ index the $UV$ positions and the elements of $A \in \R^{D \times Z}$ are i.i.d. as $\mathcal{N}(0,1/C)$. The output layer is composed again of a 1-by-1 convolution which is followed by global space averaging. That is: $\widehat{y} = \frac{1}{UV} \sum_{p=1}^{UV} G\xdt_{T,p}$, here again $p$ index the $UV$ positions and the elements of $G \in \R^{Y \times D}$ are i.i.d. as $\mathcal{N}(0,1/D)$. We refer to this convolutional model with $\phi=\tanh$ as $\overline{\overline{\mathcal{C}}}_{\tanh}$.

\subsection{Correctness checks}\label{sec:numerical_correct}

We start with a numerical study of the correctness of the results of \cref{sec:residual_diffusions}, \cref{sec:residual_diffusion_gradient} and \cref{sec:doubly_infinite}. We consider $\mathcal{F}_{\tanh}$ with $\sigma_w^2=\sigma_b^2=1$ and two 1-dimensional inputs $z^{(1)} = 0$, $z^{(2)} = 1$, hence $\xdt^{(1)}_{0,\bullet} = z^{(1)},\xdt^{(2)}_{0,\bullet} = z^{(2)}$, and simulate $10.000$ draws of the first dimension ($d = 1$) of:
\begin{enumerate}[a)]
    \item $\xdt^{(1)}_T$, $\xdt^{(2)}_T$ via the ResNet recursion \cref{eq:shallow_block};
    \item $x^{(1)}_T$, $x^{(2)}_T$ via the discretization \cref{eq:euler_sde} of the limiting SDE \cref{eq:doubly_infinite_fc};
    \item $x^{(1),\infty}_T$, $x^{(2),\infty}_T$ via the analytical transition density \cref{eq:fc_iid_infty_analytical}.
\end{enumerate}
for $L=D=500$. We only consider the first dimension because, as observed in \cref{sec:doubly_infinite_diff}, in the limit of $D \uparrow \infty$ the dimensions are i.i.d. Our analysis imply that a) and b) are equivalent when $L \uparrow \infty$, and c) is equivalent to b) when additionally $D \uparrow \infty$. As both $D$ and $L$ are large we expect good agreement between the distributions corresponding to a) b) and c). Numerical results are reported in \cref{fig:sanity_check} where indeed a good agreement with the theory is observed.
\begin{figure}[!htb]
    \centering
    \includegraphics[width=\textwidth]{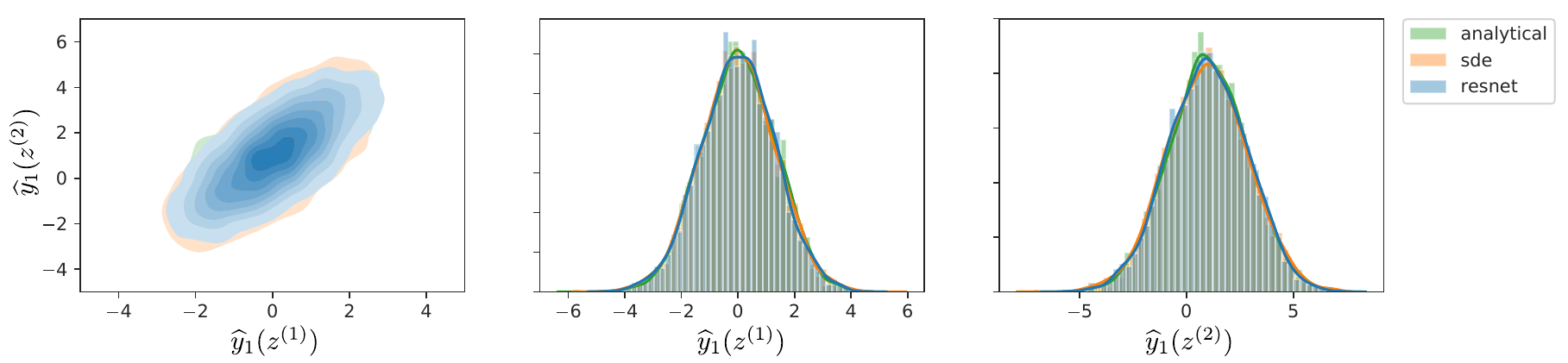}
    \caption{For model $\mathcal{F}_{\tanh}$: 2D KDE (kernel density estimator) plot for $(\widehat{y}_1(z^{(1)}),\widehat{y}_1(z^{(2)}))$ (left), 1D KDE and histogram plots for $\widehat{y}_1(z^{(1)})$ (center), $\widehat{y}_1(z^{(2)})$ (right) when $\widehat{y}_1$ is sampled from a ResNet (resnet), from the Euler discretization of its limiting SDE for $L \uparrow \infty$ (sde) and from the analytical SDE transition density for $L,D \uparrow \infty$ (analytical); $\widehat{y}$ denotes a generic model output, hence $\widehat{y}_1$ is its first dimension.}
    \label{fig:sanity_check}
\end{figure}

For the same neural network model $\mathcal{F}_{\tanh}$, \cref{fig:sanity_check_ntk} displays the convergence of the NTKs $\mathcal{K}_W^{(1,2)},\mathcal{K}_b^{(1,2)}$ to their limits $\mathcal{K}_W^{(1,2),\infty},\mathcal{K}_b^{(1,2),\infty}$ for $z^{(1)} = 1$, $z^{(2)} = 2$. The convergence is assessed in the setting where both the depth $L$ and the dimension $D$ grow unbounded jointly. Results displayed in \cref{fig:sanity_check_ntk} support the numerical analysis of \cref{sec:doubly_infinite_diff_grad}. Results also support the conjecture that the order in which the limits are taken does not impact the results for the smooth activation functions considered in the present work.
\begin{figure}[!htb]
    \centering
    \includegraphics[width=0.6\textwidth]{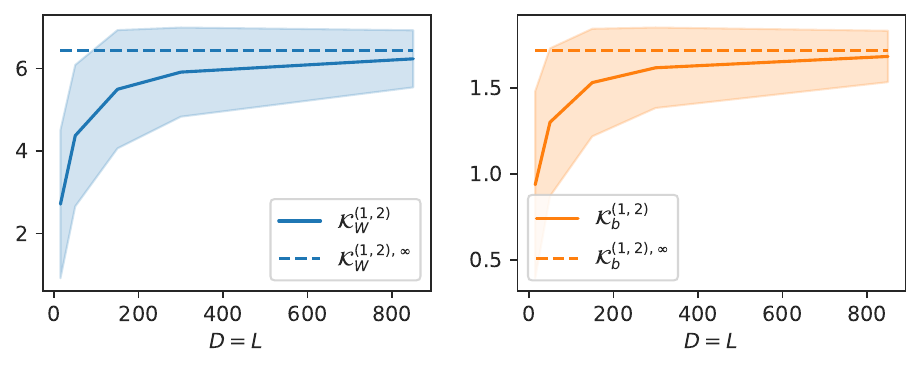}
    \caption{For model $\mathcal{F}_{\tanh}$: NTKs $\mathcal{K}_W^{(1,2)},\mathcal{K}_b^{(1,2)}$ as function of $L=D$ and their analytical limits $\mathcal{K}_W^{(1,2),\infty},\mathcal{K}_b^{(1,2),\infty}$ corresponding to $D,L \uparrow \infty$; empirical average plotted with solid line, shaded areas correspond to $\pm2$ empirical standard deviations.}
    \label{fig:sanity_check_ntk}
\end{figure}

\subsection{Function space distributions}

We show empirically that the dependency on the input is retained, and that the output distribution does not exhibit perfect correlation for very deep residual networks constructed as in the present paper. Again, we consider the neural network model $\mathcal{F}_{\tanh}$ with $\sigma_w^2=\sigma_b^2=1$. \cref{fig:sanity_check} shows that $\xdt^{(1)}_{T,1}$ and $\xdt^{(2)}_{T,1}$ have different distributions. This means that the input dependency is retained in the neural network. Furthermore, from the left plot we see that $\xdt^{(1)}_{T,1}$ and $\xdt^{(2)}_{T,1}$ are not perfectly correlated, otherwise the 2D KDE would collapse to a straight line.

\cref{fig:f_diffusion} (top panels), which can be contrasted with \cref{fig:fspace_T_example}, displays samples of $\xdt_{T,1}$ from $\mathcal{F}_{\tanh}$ in function space for different combinations of $L$ and $D$. More specifically, we approximate function draws by considering 400 inputs $z^{(i)}$ equally spaced on $[-2,2]$. Using the ResNet recursion \cref{eq:shallow_block} we obtain 400 output values $\xdt_{T,1}^{(i)}$. We repeat this procedure to obtain 10.000 function draws and report the results in \cref{fig:f_diffusion} (top). For $L=D=500$, i.e. for jointly large width and depth, the function draws are close to linear in agreement with \cref{sec:doubly_infinite_diff_grad}. We then replicate this experiment for $\mathcal{F}_{\swish}$ and we report the results in \cref{fig:f_diffusion} (bottom panels). In this case $\phi'(0) = \phi''(0) = 1/2$ and \cref{ass:non_explosivity} is not satisfied, but in this specific instance we did not observe divergent trajectories for the $10.000$ function draws. The impact of adding an input adaptation layer is limited to symmetrizing the function space distributions around the origin, while $\mathcal{F}_{\tanh}$ and $\mathcal{F}_{\swish}$ trend upward with $z$. Hence, we do not include additional plots for this additional case as they add little information. \cref{app:additional_plots} contains additional 2D plots of samples of $\xdt_{T,1}$ for both $\mathcal{F}_{\tanh}$ and $\mathcal{F}_{\swish}$.
\begin{figure}
    \centering
    \includegraphics[width=\linewidth]{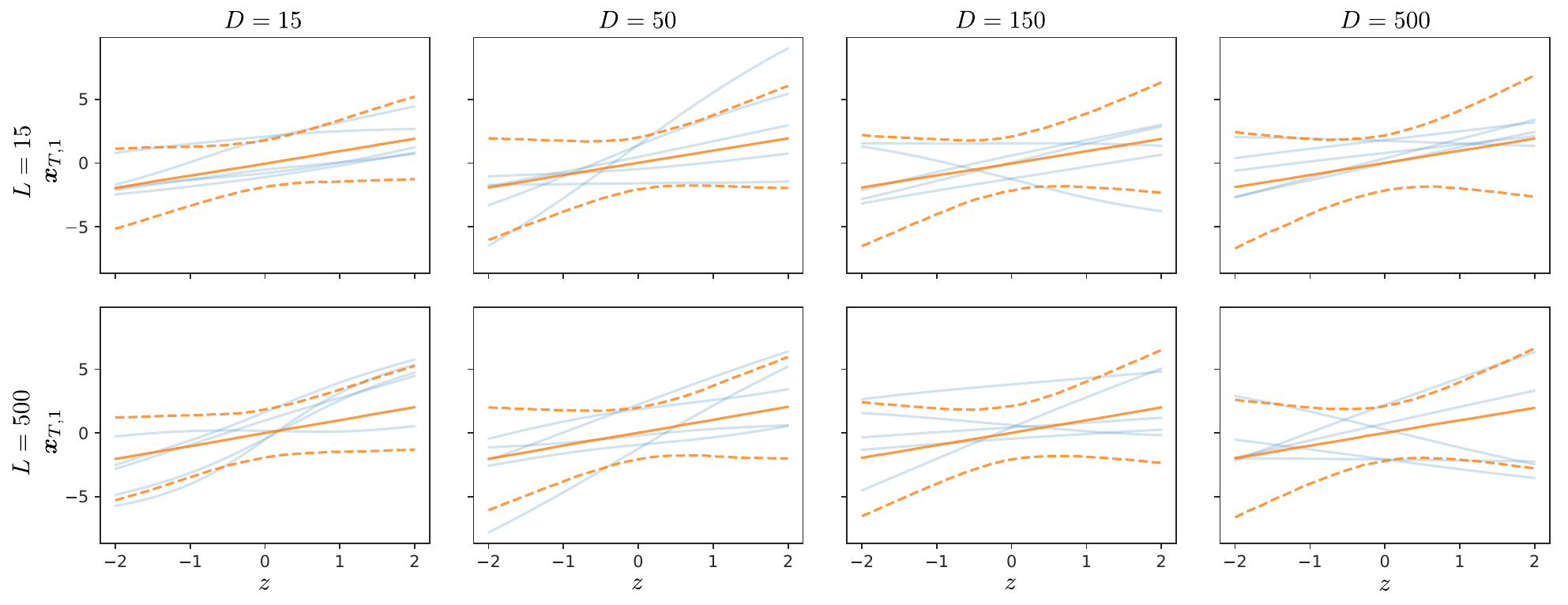}
    \includegraphics[width=\linewidth]{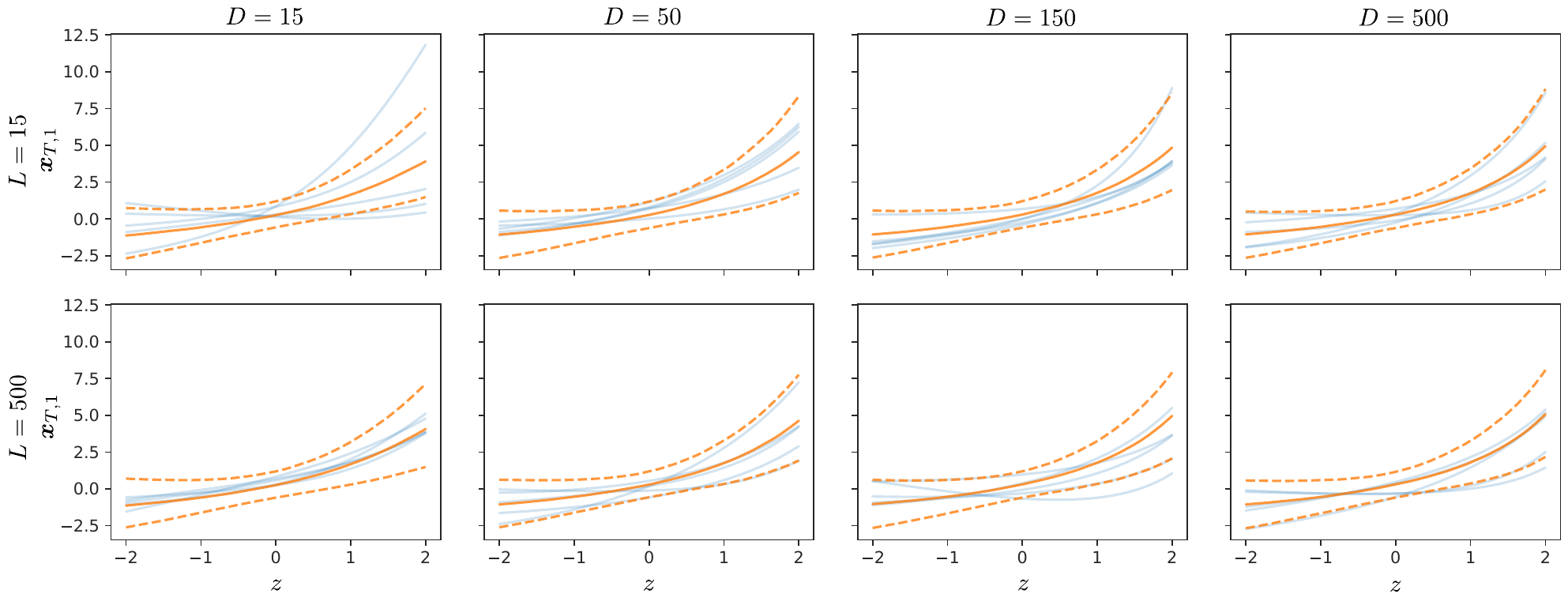}
    \caption{Function samples of $\xdt_{T.1}$ for $\mathcal{F}_{\tanh}$ (top) and ${F}_{\swish}$ (bottom), see \cref{fig:fspace_T_example} for the description of the plotted quantities.}
    \label{fig:f_diffusion}
\end{figure}

Finally, \cref{fig:corr_sde} (top panels) displays the correlations $\rho[\xdt^{(1)}_{T,1},\xdt^{(2)}_{T,1}]$ for the neural network's inputs $(z^{(1)},z^{(2)})$ in the range $[-2,2] \times [-2,2]$, for the $\tanh$ and $\swish$ activation functions: for different inputs, the corresponding output correlations are far from 1. We refer to the model of \cref{fig:fspace_T_example} with $\tanh$ activation as $\mathcal{E}_{\tanh}$, and to the model of \cref{fig:fspace_T_example} with $\ReLU$ activation as $\mathcal{E}_{\ReLU}$. For the sake of comparison, we show in \cref{fig:corr_sde} (bottom panels) the correlations $\rho[x^{(1)}_{last,1},x^{(2)}_{last,1}]$ for pre-activation 1 for $\mathcal{E}_{\tanh}$ and $\mathcal{E}_{\ReLU}$: all correlations are close to 1.
\begin{figure}
    \centering
    \includegraphics[width=0.65\linewidth]{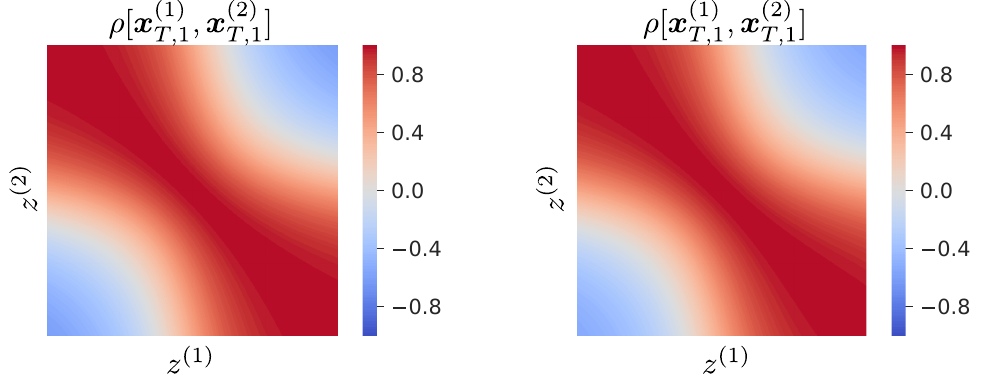}
    \includegraphics[width=0.65\linewidth]{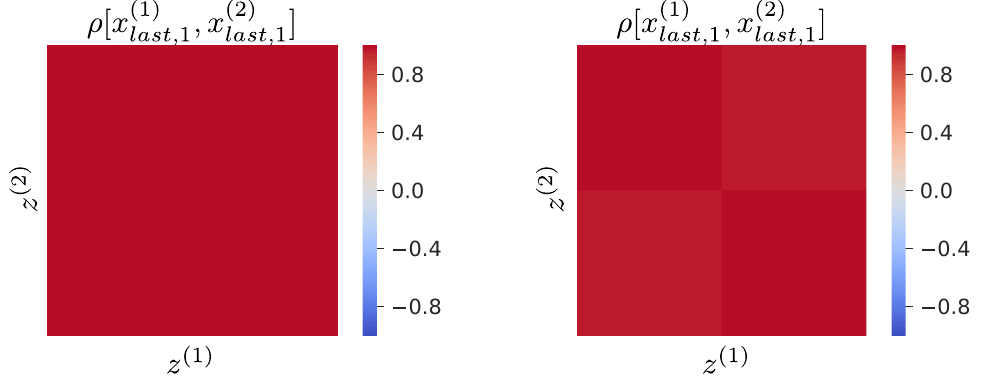}
    \caption{Output correlation heatmap for $\mathcal{F}_{\tanh}$ (top-left), $\mathcal{F}_{\swish}$ (top-right),  $\mathcal{E}_{\tanh}$ (bottom-left), $\mathcal{E}_{\ReLU}$ (bottom-right).}
    \label{fig:corr_sde}
\end{figure}

\subsection{Doubly infinite fully trained fully-connected ResNets}\label{sec:numerical_ft_fc}

We consider the MNIST dataset \citep{lecun1998mnist}. In particular, each observation $(z,k)$ is composed of an image $z$ and a target $k$ among 10 classes representing the numbers $0$ to $9$. We flatten the images obtaining $z \in \R^{784}$ and, as common, we rescale each $z$ as $z/255$ to bound the inputs on $[0,1]$. We consider the neural network $\overline{\overline{\mathcal{F}}}_{\tanh}$ trained via full-batch GD training and average MSE loss. In order to frame classification as a regression problem, we use 1-hot encoding: each class $k=1,\dots,10$ is encoded as $y_k \in \R^{10}$ which has the $k$-th component equal to $1$ and all other components equal to $0$. The gradients are computed with respect to $(\varepsilon^W_t, \varepsilon^b_t)$ in \cref{eq:fc_iid_discr_std1} and \eqref{eq:fc_iid_discr_std2}.

We consider $\overline{\overline{\mathcal{F}}}_{\tanh}$ with $\sigma_w^2=1$ and $\sigma_b^2=0.1^2$. The use of a smaller bias variance is common in the NTK literature \citep{arora2019exact}. From \cref{sec:doubly_infinite_diff_complete} we know that as $L$ and $D$ increase the fully trained $\overline{\overline{\mathcal{F}}}_{\tanh}$ collapses to noiseless Bayesian linear regression. We consider 20.000 randomly sampled observations from the training portion of the MNIST dataset, and we compute the test accuracy on the test portion of the MNIST dataset, which is composed of 10.000 observations. Using 1-hot encoding we perform kernel regression using kernel \cref{eq:doubly_infinite_ntk_2} via standard kernel regression \citep{williams2006gaussian} for the predictive posterior mean of Gaussian processes. For numerical stability the model is augmented with a small noise variance equal to $1/20.000$, and we obtain a test accuracy of $85.36\%$. We compare this accuracy with test accuracies computed for $\overline{\overline{\mathcal{F}}}_{\tanh}$ under different values of $D=L$, which is fully-trained for 120 epochs. We use a single learning rate tuned to optimize final test accuracy. 

In practice, the training of neural networks typically is performed via SGD, or via other stochastic variants of GD, as full-batch training is prohibitively expensive for large datasets. Accordingly,  here we perform SGD training of $\overline{\overline{\mathcal{F}}}_{\tanh}$, with batches of 200 observations each. Again, we consider 120 epochs and different different values of $D=L$. The same learning rate is used. In both experiments no further adjustments are performed, such as gradient clipping. The results are reported in \cref{fig:nkt_fc_train} (left). We observe that there is strong agreement between the limiting theoretical test accuracy and the final test accuracy of $\overline{\overline{\mathcal{F}}}_{\tanh}$ fully trained with GD, which is the case covered by our theory. Moreover this result empirically extends to SGD.
\begin{figure}
    \centering
    \includegraphics[width=0.3\linewidth]{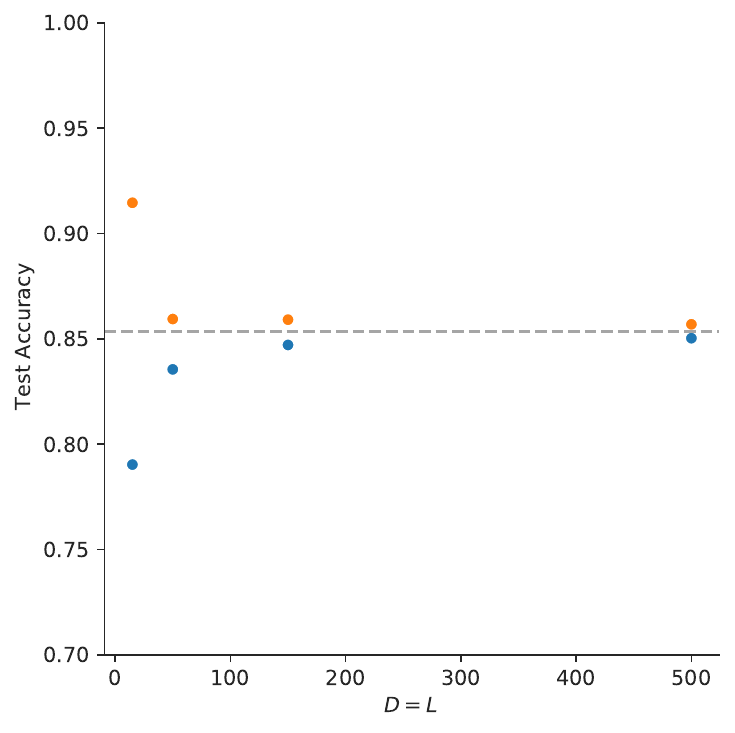}
    \includegraphics[width=0.3\linewidth]{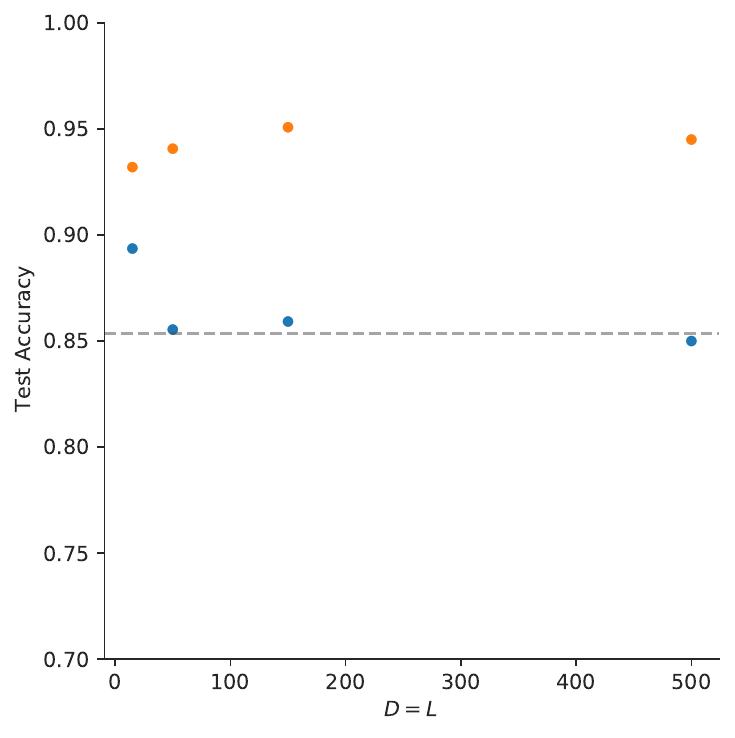}
    \caption{Final MNIST test accuracy after 120 epochs of training $\overline{\overline{\mathcal{F}}}_{\tanh}$ via GD (left, blue), SGD (left, orange), full-batch Adam (right, blue), mini-batch Adam (right, orange) compared with the theoretical limiting value corresponding to Bayesian linear regression (dashed gray).}
    \label{fig:nkt_fc_train}
\end{figure}

For completeness, we consider the same training setting with Adam \citep{kingma2015adam}, a popular adaptive stochastic optimizer, and we report the results in \cref{fig:nkt_fc_train} (right) for both full-batch and mini-batch variants. While there is no strong consensus on whether Adam outperforms or underperforms SGD when a carefully tuned learning rate is used \citep{wilson2017marginal,choi2019empirical}, Adam is known to be less sensitive to learning rate specifications and exhibits more robust behavior in difficult optimization problems. In particular, the proposed experiment provides an alternative viewpoint: Adam (with mini-batching, as standard) is able to ``escape'' the domain of attraction of linear model solutions, at least up to the largest model size here considered. We suspect that more complex neural network architectures might exhibit analogous pathologies at initialization when the number of network's parameters is very large, and Adam seems more robust to these issues. In any case, a formal investigation would require new results in the NTK literature to cover adaptive optimizers.

\subsection{Doubly infinite fully trained convolutional ResNets}\label{sec:numerical_ft_cnn}

While a theoretical investigation of the backward properties of CNNs is beyond the scope of the present paper, in this section we empirically investigate to what extent the observations of \cref{sec:numerical_ft_fc} extends to convolutional neural networks. In particular, we consider $\overline{\overline{\mathcal{C}}}_{\tanh}$ with $\sigma_w^2=1$ and $\sigma_b^2=0.1^2$. The setting is the same of \cref{sec:numerical_ft_fc}, with the exception that the input images are not flattened. We consider training under MSE loss for 120 epochs with both SGD and Adam. For computation reasons we restrict the maximum model size to $D=L=150$ and do not investigate full-batch training. We report the results of this experiment in \cref{fig:nkt_cnn_train}. These results suggests convergence as $D=L$ which is reminiscent of what observed in \cref{fig:nkt_fc_train}. Similarly to \cref{fig:nkt_fc_train} Adam exhibits superior performance for large $D=L$.

\begin{figure}
    \centering
    \includegraphics[width=0.3\linewidth]{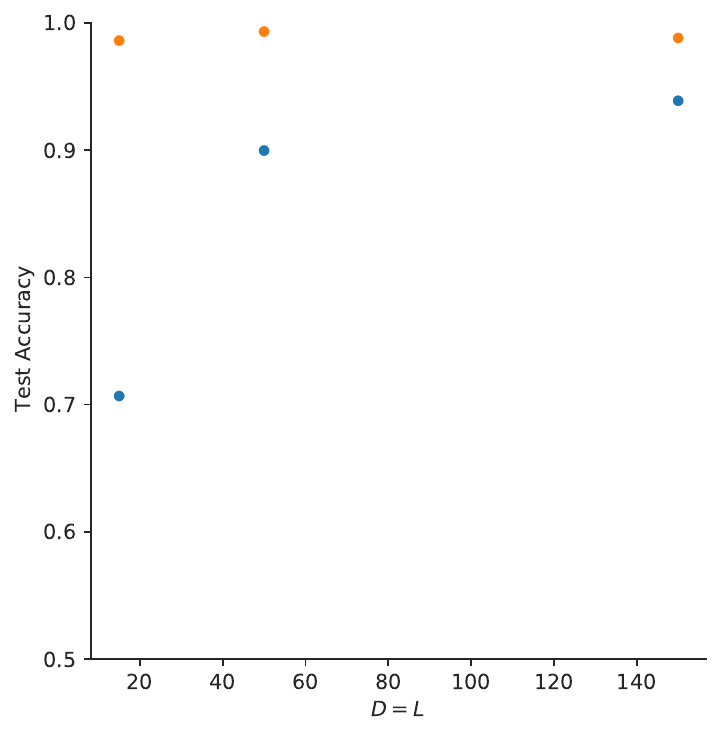}
    \caption{Final MNIST test accuracy after 120 epochs of training $\overline{\overline{\mathcal{C}}}_{\tanh}$ with SGD (blue) and Adam (orange).}
    \label{fig:nkt_cnn_train}
\end{figure}

\subsection{Empirical Bayes for wide and deep fully-connected ResNets}\label{sec:empirical_bayes}

In this section we present a proof of concept experiment where hyperparameter optimization is performed via empirical Bayes on a surrogate Bayesian linear model (BLM). In \cref{sec:doubly_infinite_diff_complete} we established that the prior of a ResNet with fully i.i.d. parameterization and the identity $\psi$ activation converges to a BLM as width and depths grows unboundedly. We can thus consider the marginal likelihood of the limiting BLN as a proxy for the marginal likelihood of a wide and deep ResNet and leverage on it to conduct approximate empirical Bayes inference for the finite ResNet. More in detail, we consider a fully connected ResNet with the input adaptation layer detailed in \cref{sec:doubly_infinite_diff_complete}. The corresponding limiting kernel in the doubly infinite case is given by \cref{eq:doubly_infinite_cgk}. We consider a simple classification task consisting of discriminating between the digits '3' and '7' on the MNIST dataset. We randomly sample 100 images corresponding to the classes '3' and '7', which we will use to maximize the marginal likelihood. Each greyscale image $z$ is mapped to the $[-1,+1]$ interval via the transform \texttt{z -> z/255*2-1}. The '3' class is mapped to $-1$ while '7' is mapped to $+1$. We assume a small i.i.d. Gaussian observation noise ($\sigma_e=0.01$) to improve numerical stability. The marginal likelihood of the limiting BLM is optimized via a differential evolution algorithm \citep{price2013differential} over the positive hyperparameters $\sigma_z,\sigma_w,\sigma_b$. We obtain an average, i.e. per sample, negative marginal log-likelihood (NLL) of $0.65$ which suggests a good fit to the data: for comparison the ``default'' hyperparameters $\sigma_z=\sigma_w=\sigma_b=1.0$ result in a NLL of $3.25$. In \cref{fig:marginal_likel} we plot the optimal and ``default'' NLL of the BLM jointly with the NLL of the ResNet for different widths and depths. The NLL of the finite ResNet is naïvely estimated via Monte Carlo integration over the hidden layers' units, with the last layer integrated out analytically. This estimation can be considered reliable only for large width, where the stochastic dependency of each layer on the previous layer becomes deterministic \citep{lee2018deep}. Nonetheless, this simple approach suffices to illustrate the convergence of the finite width and depth NLLs to their limiting counterparts.

\begin{figure}
    \centering
    \includegraphics[width=0.35\linewidth]{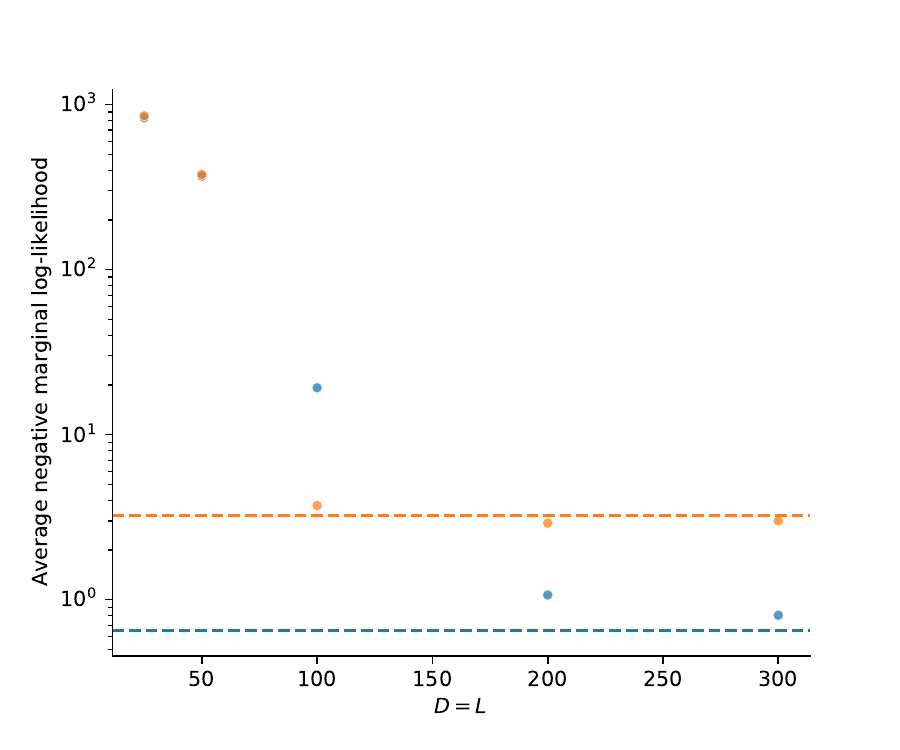}
    \caption{NLL of the limiting BLM (dashed line) and of the corresponding ResNet for different widths and depths (dots); orange: ``default'' hyperparameters, blue: optimal hyperparameters.}
    \label{fig:marginal_likel}
\end{figure}

\section{Discussion}\label{sec:conclusion}

We investigated the large-depth limit of identity ResNets \citep{he2016identity}, establishing convergence to solutions of SDEs. Our results rely on smooth activation functions and on distributions for network's parameters which shrink as the total depth increases; further conditions on activation functions are obtained by restricting the limiting SDEs to be non explosive. Building upon our connection between infinitely deep fully-connected ResNets and diffusion processes, we showed that both forward and backward dynamics are well-behaved. With regards to forward-propagation, we showed that as the total depth grows unboundedly: i) the dependency of the last layer on the input does not vanish; ii) the last layer, as a stochastic function on input space, remains flexible without collapsing to restrictive families of distributions, iii) the last layer does not collapse to a deterministic limit, nor it diverges. Moreover, we established conditions on the activation functions in order to avoid explosive dynamics over the layers. With regards to backward-propagation, under non-explosivity conditions, we showed that the Jacobian of the final layer with respect to any layer can be expressed as the multiplication of two matrix diffusions which satisfy the same desiderata i), ii) and iii) in the large-depth limit, and hence it is well-behaved. Moreover, we addressed the problem of the trainability at initialization of such a neural network, showing that exploding gradients are not possible in the large-width limit, and that the ResNet is invertible. In contrast to the information propagation approach, our analysis covers finitely-wide neural networks and correlated network's parameters. While we focussed on fully-connected ResNets, there are no theoretical issues to extend our results to the more general convolutional architectures, as we have shown for the forward-propagation analysis of convolutional  ResNets.

Limited to fully i.i.d. network's parameters and fully-connected neural networks without the second activation $\psi$, we investigated the case of the doubly infinite ResNets where both the network's depth and the network's width grow unboundedly. The attractiveness of the doubly infinite setting is mainly related to the potential of obtaining analytical results. We showed that doubly infinite fully-connected ResNets converge to Gaussian processes whose kernels can be computed by solving corresponding systems of ODEs. In particular, when $\phi''(0) = 0$ and the model is completed with a fully i.i.d. input adaptation, we showed that the doubly infinite ResNet collapses to a Bayesian linear model with a fully factorial prior distribution. To conclude, we obtained the form of the NTK that corresponds to full training with continuous time gradient descent and quadratic loss of doubly infinite ResNets. In particular, we observed that such a kernel is qualitatively identical to the kernel of the Gaussian process arising in this doubly infinite limit, thus implying that fully trained doubly infinite networks of the considered class are again equivalent to performing linear regression. Although our results on the doubly infinite setting are not completely mathematically rigorous, numerical experiments support the validity of the proposed derivations. Our study on doubly infinite ResNets illustrate the pitfalls that must be overcome in order to derive non-trivial limits. Architectures, network's parameters and activation functions need to satisfy precise conditions. However, still under these conditions, the resulting limiting behavior can be very unexpressive. While this is an undesirable result if inference via the limiting process is the final goal, the connection to simple models allows to perform hyper-parameter optimization on the finitely-sized neural network by means of empirical Bayes procedures on the corresponding linear model. 

The fields of  diffusion processes and SDEs are mature and rich fields \citep{oksendal2003stochastic,karatzas1999brownian,revuz1999continuous,kloeden1992numerical,stroock2006multidimensional}, with a vast range of theoretical results and simulation methods. We envision that examining neural networks properties from the point of view of SDEs will bring further insights. Our study suggests two main research directions of future work. Firstly, to overcome modeling limitations one could narrow the gap between theory and practice by considering more realistic residual blocks consisting of multiple layers. This may be approached either via fractional Brownian motions \citep{biagini2008stochastic} or via re-scaled Brownian motions; such an extension would allow to consider neural networks which are infinitely wide only in the residual blocks internal dimension. Secondly, a mathematically rigorous treatment of of doubly infinite setting of \cref{sec:doubly_infinite} could be developed. This, as a main difficulty, involves to  work with infinite-dimensional stochastic processes \citep{prato2014stochastic}. Indeed, the standard approach to establish large-width limits consists in postulating both the limiting and the converging processes on the same infinite space, while limiting the connectivity of the converging processes \citep{matthews2018gaussian}. This approach allows to establish convergence on a space of fixed (infinite) dimension. While the standard theory of diffusion limits is well established \citep{stroock2006multidimensional}, it covers only the case of diffusion processes of finite-dimensions. 

\clearpage
\begin{appendices}
\crefalias{section}{appsec}

\section{Proofs} \label{app:proofs}

This appendix contains all the proofs of the theorems stated in the main text and the lemmas required to prove them.

\begin{proof}[Proof of \cref{thm:sde_convergence}]
This is \cite[Theorem 2.2]{nelson1990arch}: \cref{ass:inf_coeff} and the postulated weakly unique and non-explosive weak solution satisfy all the conditions required for the application of \cite[Theorem 2.2]{nelson1990arch}.
Note that we use a stronger non-explosivity condition (\cite{oksendal2003stochastic}).
Alternatively, for this standard result the reader can refer to the monograph \cite{stroock2006multidimensional} on which \cite{nelson1990arch} is based; yet another reference is \cite{ethier2009markov}.
\end{proof}

\begin{lemma} \label{lem:bounds}
If $\phi$ satisfies \cref{ass:activation}, $\epsilon \sim \mathcal{N}(0, \sigma^2)$ with $\sigma^2 \leq \sigma_*^2$, $\alpha > 0$, then we can find $M_2(\alpha, \sigma_*^2) < \infty$ and $M_3(\alpha, \sigma_*^2) < \infty$ such that:
\begin{align*}
    &\E \left[ |\phi''(\epsilon)|^\alpha \right]  \leq M_2(\alpha, \sigma_*^2) \\ 
    &\E \left[ |\phi'''(\epsilon)|^\alpha \right] \leq M_3(\alpha, \sigma_*^2)
\end{align*}
\end{lemma}
\begin{proof}
We prove the result only for $\phi''(\epsilon)$, the case for $\phi'''(\epsilon)$ being identical.
Let $L$ large enough such that $|\phi''(x)| \leq K_1 e^{K_2 |x|}$ for $|x| \geq L$ then:
\begin{align*}
    \E \left[ |\phi''(\epsilon)|^\alpha \right] &= \E \left[ |\phi''(\epsilon)|^\alpha \1_{|\epsilon| \leq L} \right] + \E \left[ |\phi''(\epsilon)|^\alpha \1_{|\epsilon| > L} \right] \\
                                                   &\leq \sup_{|x| \leq L}|\phi''(x)|^\alpha + K_1^\alpha \E[e^{K_2 \alpha |\epsilon|}]
\end{align*}
The first term is finite, that the second one can be bounded by a finite and increasing function in $\sigma^2$ follows from the symmetry in law of $\epsilon$ and the form of its movement generating function.
\end{proof}

\begin{proof}[Proof of \cref{thm:resnet_fc_sde}]
We suppress the dependency on $t$ of vector and matrices and the conditioning in expectations and covariances in this proof to ease the notation. We also drop the boldness of $x_t$ as no confusion arises in this setting. We instead reserve subscripts for indexing: for example $x_d$ denotes the $d$-th element of a vector $x$.

Let $h = (\mu^W \sqrt{\Dt} + \varepsilon^W) \psi(x) + (\mu^b \sqrt{\Dt} + \varepsilon^b)$ so that $h \sqrt{\Dt} = \Delta W \psi(x) + \Delta b$. By second order Taylor expansion of $\phi$ around 0 we have for $d = 1, \dots, D$
\begin{equation*}
     \frac{\Delta x_d}{\Dt} = \frac{\phi(h_d \sqrt{\Dt})}{\Dt} = \phi'(0) h_d \Dt ^ {-1/2} + \frac{1}{2} \phi''(0) h_d ^ 2 + \frac{1}{6} \phi'''(\vartheta_d) h_d ^ 3 \Dt ^ {1/2}
\end{equation*}
with $\vartheta_d \in (-h_d \sqrt{\Dt}, h_d \sqrt{\Dt})$. To prove \cref{eq:mu_x} we want to show that $\forall R > 0$
\begin{equation*}
    \lim_{\Dt \downarrow 0} \sup_{\norm{x} < R} \left| \mux(x)_d - \E\left[\phi'(0) h_d \Dt ^ {-1/2} + \frac{1}{2} \phi''(0) h_d ^ 2 + \frac{1}{6} \phi'''(\vartheta) h_d ^ 3 \Dt ^ {1/2} \right] \right| = 0.
\end{equation*}
Now, $h_d = (\mu^W_d \sqrt{\Dt} + \varepsilon^W_d) \psi(x) + \mu^b_d \sqrt{\Dt} + \varepsilon^b_d$ and the distribution assumptions on $\varepsilon^W$ and $\varepsilon^b$ lead to
\begin{align*}
    &\E\left[\phi'(0)h_d\Dt^{-1/2} + \frac{1}{2} \phi''(0) h_d ^ 2 \right]\\
    &= \phi'(0)(\mu^b_d + \mu^W_d \psi(x))\\
    &\quad+ \frac{1}{2} \phi''(0) \V[\varepsilon^W \psi(x) + \varepsilon^b]_{d,d}\\
    &\quad+ \frac{1}{2} \phi''(0) \left(\mu^b_d + \mu^W_d \psi(x)\right)^2\Dt \\
    &= \mux(x)_d + \frac{1}{2} \phi''(0) \left(\mu^b_d + \mu^W_d \psi(x)\right)^2\Dt.
\end{align*}
It remains to show that
\begin{equation*}
    \lim_{\Dt \downarrow 0} \sup_{\norm{x} < R} \left| \left(\mu^b_d + \mu^W_d \psi(x)\right)^2 \right| \Dt = 0,
\end{equation*}
which holds as $\psi$ is locally bounded, and that
\begin{equation*}
    \lim_{\Dt \downarrow 0} \sup_{\norm{x} < R} \left| \E\left[ \phi'''(\vartheta_d) h_d ^ 3 \right] \right| \Dt ^ {1/2} = 0,
\end{equation*}
for which it suffices to show that $\sup_{\norm{x} < R} \left| \E\left[ \phi'''(\vartheta_d) h_d ^ 3 \right] \right|$ can be bounded by $M(R) < \infty$ uniformly in $\Dt$. By Cauchy–Schwarz
\begin{equation}
    \sup_{\norm{x} < R} \left| \E\left[ \phi'''(\vartheta_d) h_d ^ 3 \right] \right| \leq \sup_{\norm{x} < R} \E\left[ \phi'''(\vartheta_d) ^2 \right]^{1/2} \sup_{\norm{x} < R} \E\left[h_d ^ 6 \right]^{1/2}.
\end{equation}
Again, as $\psi$ is locally bounded the constraint $\sup_{\norm{x} < R}$ corresponds to a constraint on the variance of $h_d$ hence the second $\sup$ is finite. By \cref{lem:bounds} the first $\sup$ is finite too and not increasing in $\Dt$ as $|\vartheta_d| \leq \sqrt{\Dt}|h_d|$ which allows us to produce the desired bound $M(R)$.

Regarding \cref{eq:continuity_x}, by first order Taylor expansion of $\phi$ around 0 we need to show that for $d = 1, \dots, D$ and $R>0$
\begin{equation*}
    \lim_{\Dt \downarrow 0} \sup_{\norm{x} < R} \left| \E\left[ \frac{\left(\phi'(0) h_d \Dt^{1/2} + \frac{1}{2} \phi''(\vartheta_d) h_d ^ 2 \Dt \right)^4}{\Dt} \right] \right| = 0
\end{equation*}
with $\vartheta_d \in (-h_d \sqrt{\Dt}, h_d \sqrt{\Dt})$. Note that the term inside the expectation is composed of a sum of terms of the form $k h_d^n \phi''(\vartheta_d)^m \Dt^\alpha$ for integers $n,m \geq 0$ and reals $\alpha > 0$, $k \in \mathbb{R}$. This results from repeated applications of the Cauchy–Schwarz inequality and \cref{lem:bounds} as we did previously to prove \cref{eq:mu_x}.

Regarding \cref{eq:sigma_x}, we can compute $\E[\Delta x (\Delta x)\tran]/\Dt$ instead of $\V[\Delta x]/\Dt$ as in the infinitesimal limit of $\Dt \downarrow 0$ the two quantities have to agree due to the convergence of the infinitesimal mean that we have already established. Hence by first order Taylor expansion of $\phi$ around 0 we need to show that for $d,u = 1, \dots, D$ and $R>0$:
\begin{align*}
    & \lim_{\Dt \downarrow 0} \sup_{\norm{x} < R} \Bigg| \sigmax^2(x)_{d,u}\\
    &- \E\left[ \frac{\big(\phi'(0) h_d \Dt^{1/2} + \frac{1}{2} \phi''(\vartheta_d) h_d ^ 2 \Dt \big) \big(\phi'(0) h_u \Dt^{1/2} + \frac{1}{2} \phi''(\vartheta_u) h_u ^ 2 \Dt \big)}{\Dt} \right] \Bigg| = 0
\end{align*}
with $\vartheta_d \in (-h_d \sqrt{\Dt}, h_d \sqrt{\Dt})$, $\vartheta_u \in (-h_u \sqrt{\Dt}, h_u \sqrt{\Dt})$. The only term inside the expectation not vanishing in $\Dt$ is 
\begin{align*}
    &\E[\phi'(0)^2 h_d h_u]\\
    &= \phi'(0)^2 \V[\varepsilon^W \psi(x) + \varepsilon^b]_{d,u} + \phi'(0)^2 \left(\mu^b_d + \mu^W_d \psi(x)\right)\left(\mu^b_u + \mu^W_u \psi(x)\right) \Dt\\
    &= \sigmax^2(x)_{d,u} + \phi'(0)^2 \left(\mu^b_d + \mu^W_d \psi(x)\right)\left(\mu^b_u + \mu^W_u \psi(x)\right) \Dt.
\end{align*}
The (uniform on compacts) convergence of all terms aside from $\sigmax^2(x)_{d,u}$ to $0$ once again follows from repeated applications of the Cauchy–Schwarz inequality and \cref{lem:bounds}.

Now, the continuity of $\mux(x)$ and $\sigmax(x)$ are a consequence of the continuity of the conditional covariance $\V[\varepsilon^W \psi(x) + \varepsilon^b]$, and as $\V[\varepsilon^W \psi(x) + \varepsilon^b]$ is positive semi-definite so is $\sigmax^2(x)$. Hence all the conditions of Assumption \cref{ass:inf_coeff} hold true.

Finally, as $\psi$ is differentiable two times with continuity, it follows from the dependency of $\mux$ and $\sigmax^2$ on $x$ only through $\V[\varepsilon^W \psi(x) + \varepsilon^b]$ that \cref{ass:existence_uniqueness} is satisfied too. The application of \cref{thm:sde_convergence} completes the proof.
\end{proof}

\begin{proof}[Proof of \cref{thm:resnet_fc_sde_2}]
Notice that
\begin{equation*}
    d[W \psi(x)]_t + d[b]_t = d[W \psi(x) + b]_t = \diag(\V[\varepsilon^W_t \psi(x_t) + \varepsilon^b_t|x_t])dt
\end{equation*}
Then expanding $dW_t$ and $db_t$ in \cref{eq:resnet_fc_sde_2} shows that the drift terms are matched between \cref{eq:resnet_fc_sde} and \cref{eq:resnet_fc_sde_2}. The quadratic variation of \cref{eq:resnet_fc_sde} is
\begin{equation*}
    \phi'(0)^2 \diag(\V[\varepsilon^W_t \psi(x_t) + \varepsilon^b_t|x_t]) dt
\end{equation*}
which is equal to the quadratic variation of \cref{eq:resnet_fc_sde_2} as it is computed as
\begin{equation*}
    d[x]_t = d[\phi'(0)(W \psi(x) + b)]_t = \phi'(0)^2 d[W \psi(x) + b]_t
\end{equation*}
This shows the equivalence in law between the solution of \cref{eq:resnet_fc_sde} and the solution of \cref{eq:resnet_fc_sde_2}. Then \cref{eq:resnet_fc_sde_2_cross} immediately follows by direct computation.
\end{proof}

\begin{proof}[Proof of \cref{thm:resnet_fc_matrix_variate} and \cref{thm:resnet_fc_iid}]
Notice that
\begin{align*}
    &d[W\psi(x^{(i)}) + b, W\psi(x^{(j)}) + b]_t\\ 
    &=\C[\varepsilon^W_t \psi(x_t^{(i)}) + \varepsilon^b_t, \varepsilon^W_t \psi(x_t^{(j)}) + \varepsilon^b_t|x_t^{(i)},x_t^{(j)}]dt\\
    &=\big(\Sigma^b + \C[\varepsilon^W_t \psi(x_t^{(i)}), \varepsilon^W_t \psi(x_t^{(j)})|x_t^{(i)},x_t^{(j)}]\big)dt
\end{align*}
and
\begin{align*}
    &\C[\varepsilon^W_t \psi(x_t^{(i)}), \varepsilon^W_t \psi(x_t^{(j)})|x_t^{(i)},x_t^{(j)}]_{r,c}\\
    &=\E[(\varepsilon^W_{t,r,\bullet} \psi(x_t^{(i)}))(\varepsilon^W_{t,c,\bullet} \psi(x_t^{(j)}))|x_t^{(i)},x_t^{(j)}]\\
    &=\sum_{d,u=1}^{D} \psi(x_{t,d}^{(i)}) \psi(x_{t,u}^{(j)}) \E[W_{r,d} W_{c,u}]\\
    &=\Sigma^{W_O}_{r,c} \sum_{d,u=1}^{D} \psi(x_{t,d}^{(i)}) \psi(x_{t,u}^{(j)}) \Sigma^{W_I}_{d,u}\\
    &=\Sigma^{W_O}_{r,c} (\psi(x_t^{(i)})\tran \Sigma^{W_I} \psi(x_t^{(j)})).
\end{align*}
This proves \cref{thm:resnet_fc_matrix_variate}. \cref{thm:resnet_fc_iid} follows by setting $\sigma^b = \sigma_b \I_D$, $\sigma^{W_I} = \I_D$ and $\sigma^{W_O} = \sigma_w D^{-1/2} \I_D$.
\end{proof}

\begin{proof}[Proof of \cref{thm:jacobian_resnet_fc}]
Here $g_t$ is a $D \times D$ matrix-valued SDE instead of standard $D$-dimensional (vector) SDEs.
All the theory presented in \cref{sec:preliminaries_diffusion} continues to hold with the obvious modifications by working on the vectorization of matrix-valued processes. When establishing the limits for $g_t$ in \cref{ass:inf_coeff} the conditioning is both on $g_t$ and on $x_t$, indeed the convergence to the limiting process is obtained jointly in $x_t$ and $g_t$. This proof follows the exact same path of the proofs of \cref{thm:resnet_fc_sde} and \cref{thm:resnet_fc_sde_2} so we highlight the main steps only. And once again we suppress the dependency on $t$ of vector and matrices, the conditioning in expectations and covariances, and the boldness of $x_t$ and $g_t$ as no confusion arises in this setting:
\begin{equation*}
    \Delta g = \big(\phi'(\Delta W \psi(x) + \Delta b) {1_D}\tran \odot \Delta W \odot 1_{D}\psi'(x)\tran \big)g
\end{equation*}

Let $h = (\mu^W \sqrt{\Dt} + \varepsilon^W) \psi(x) + (\mu^b \sqrt{\Dt} + \varepsilon^b)$ so that $h \sqrt{\Dt} = \Delta W \psi(x) + \Delta b$.
By second order Taylor expansion of $\phi'$ around 0 we have for $d = 1, \dots, D$
\begin{equation*}
    \phi'(h_d \sqrt{\Dt}) = \phi'(0) + \phi''(0)h_d\sqrt{\Dt} + \frac{1}{2}\phi'''(\vartheta_d) h_d^2 \Dt
\end{equation*}
with $\vartheta_d \in (-h_d \sqrt{\Dt}, h_d \sqrt{\Dt})$.
Then with $\vartheta = \begin{bmatrix}\vartheta_1 \cdots \vartheta_D\end{bmatrix}$
\begin{equation*}
    \Delta g = \big((\phi'(0)1_D\sqrt{\Dt} + \phi''(0)h \Dt + \frac{1}{2}\phi'''(\vartheta) h^2 \Dt^{3/2}) {1_D}\tran \odot (\mu^W \sqrt{\Dt} + \varepsilon^W) \odot 1_{D}\psi'(x)\tran \big)g
\end{equation*}
In order to obtain the instantaneous mean of $g$ we need to compute
\begin{align*}
    \E\left[\frac{\Delta g}{\Dt}\right] &= \phi'(0) (\mu^W \odot 1_{D}\psi'(x)\tran) g + \phi''(0) (\E[\varepsilon^W \psi(x) 1_D\tran \odot \varepsilon^W] \odot 1_{D}\psi'(x)\tran) g + r_{g,\mu}(g,x,\Dt)\\
    &= \mu_g(g,x) + r_{g,\mu}(g,x,\Dt)
\end{align*}
where $r_{\mu}(g,x,\Dt)$ is a vector-valued remainder term and we want to show that for each $R > 0$
\begin{equation*} \label{eq:g_mu_limit}
    \lim_{\Dt \downarrow 0}\sup_{\norm{g} + \norm{x} < R}\norm{r_{g,\mu}(g,x,\Dt)} = 0
\end{equation*}

By first order Taylor expansion of $\phi'$ around 0 we have for $d,d'=1,\dots,D$
\begin{equation*}
    \Delta g_{d,d'} = \big((\phi'(0) + \phi''(\vartheta_d)h_d\sqrt{\Dt})1_D\tran \odot \Delta W_{d,\bullet} \odot \psi'(x) \big) g_{\bullet,d'}
\end{equation*}
with $\vartheta_d \in (-h_d \sqrt{\Dt}, h_d \sqrt{\Dt})$.
In order to obtain the instantaneous covariance of $g$ we need to compute for $d,d',u,u'=1,\dots,D$
\begin{align*}
    \E\left[\frac{\Delta g_{d,d'} \Delta g_{u,u'}}{\Dt}\right] &= \phi'(0)^2 \E[(\varepsilon^W_{d,\bullet} \odot \psi'(x)) g_{\bullet,d'} (\varepsilon^W_{u,\bullet} \odot \psi'(x)) g_{\bullet,u'}] + r_{g,\sigma}(g,x,\Dt)_{d,d',u,u'}\\
    &= \sigma_g^2(g,x)_{d,d',u,u'} + r_{g,\sigma^2}(g,x,\Dt)_{d,d',u,u'}
\end{align*}
where $r_{g,\sigma^2}(g,x,\Dt)$ is a remainder term (a 4 dimensional tensor, as $\sigma_g^2(g,x)$) and we want to show that for each $R > 0$
\begin{equation*}
    \lim_{\Dt \downarrow 0}\sup_{\norm{g} + \norm{x} < R}\norm{\vect(r_{g,\sigma^2}(g,x,\Dt))} = 0
\end{equation*}

Again by first order Taylor expansion of $\phi'$ around 0 we want to prove for $d,d'=1,\dots,D$
\begin{equation*}
    \lim_{\Dt \downarrow 0}\sup_{\norm{g} + \norm{x} < R}\norm{c_{g}(g,x,\Dt)} = 0
\end{equation*}
where
\begin{equation*}
    c_g(g,x,\Dt) = \E\left[\frac{\left(\big((\phi'(0) + \phi''(\vartheta)h\sqrt{\Dt})1_D\tran \odot \Delta W \odot 1_{D}\psi'(x) \big) g \right)^4}{\Dt}\right]
\end{equation*}
(here the fourth power is element-wise) with $\vartheta = \begin{bmatrix}\vartheta_1 \cdots \vartheta_D\end{bmatrix}$ and $\vartheta_d \in (-h_d \sqrt{\Dt}, h_d \sqrt{\Dt})$ to satisfy the continuity in probability requirement.

Then the limit
\begin{equation*}
    \lim_{\Dt \downarrow 0}\sup_{\norm{g} + \norm{x} < R} \left(\norm{r_{g,\mu}(g,x,\Dt)} + \norm{\vect(r_{g,\sigma^2}(g,x,\Dt))} + \norm{c_{g}(g,x,\Dt)}\right) = 0
\end{equation*}
again follows from repeated applications of the Cauchy–Schwarz inequality and \cref{lem:bounds}.
Now, \cref{ass:existence_uniqueness} follows from the linearity of the expectation operator and the positive semi-definiteness of $\sigma_g^2(g,x)$ is easily checked.
The equivalence of \cref{eq:jacobian_resnet_fc} to the matrix-SDE defined by $\mu_g(g,x)$ and $\sigma_g^2(g,x)$ is established by comparing the drift and quadratic covariation terms.
This completes the proof.
\end{proof}

\begin{proof}[Proof of \cref{thm:jacobian_resnet_fc_inverse}]
Let $dZ_t = \left(\phi'(0)dW_t + \phi''(0)d[W \psi(x) {1_D}\tran \odot W]_t\right) \odot 1_{D}\psi'(x_t)$.
Then $g_t$ given by $dg_t = dZ_t g_t$ is the right stochastic exponential of $Z_t$ which we denote, following \cite{protter2005stochastic}, as $g = \mathcal{E}^{R}(Z)$.
Let define the (left) stochastic exponential $u = \mathcal{E}(Z)$ of $Z_t$ by $du_t = u_t dZ_t$. 
From \cite[Chapter V, Theorem 48]{protter2005stochastic} we know that $g_t$ is invertible and that
\begin{equation*}
    \mathcal{E}(Z) \mathcal{E}^R(-Z + [Z,Z]) = \I_D
\end{equation*}
It follows that
\begin{equation*}
    \mathcal{E}^R(Z) \mathcal{E}(-Z + [Z,Z]) = \I_D
\end{equation*}
hence $g^{-1} = \mathcal{E}(-Z + [Z,Z])$ which completes the proof.
\end{proof}

\begin{lemma} \label{thm:sde_layers}
When $\psi$ is the identity function SDE \cref{eq:resnet_fc_iid} is equivalent, in distribution, to the representation where each data point $i$ has an associated $D$-dimensional BM $B_t^{(i)}$, $\{B_t^{(i)}\}_{i=1}^N$ are dependent over $i$, and each $B_t^{(i)}$ corresponds to both the weights and biases sources of randomness. That is, the SDE \cref{eq:resnet_fc_iid} is equivalent to the following
\begin{align*}
    &dx_t^{(i)} = \phi'(0)(\sigma_b^2 + \sigma_w^2 q_t^{(i)})^{1/2} dB_t^{(i)} + \frac{1}{2}\phi''(0)(\sigma_b^2 + \sigma_w^2 q_t^{(i)}) 1_D dt\\
    &d[B^{(i)},B^{(j)}]_t = \frac{\sigma_b^2 + \sigma_w^2 \lambda_t^{(i,j)}}{\big((\sigma_b^2 + \sigma_w^2 q_t^{(i)})(\sigma_b^2 + \sigma_w^2 q_t^{(j)})\big)^{1/2}} \I_D dt
\end{align*}
Let $\widetilde{\mu}_x(x) = \frac{1}{2}\phi''(0)(\sigma_b^2 + \sigma_w^2 q(x))$, $\widetilde{\sigma}_x(x) = \phi'(0)(\sigma_b^2 + \sigma_w^2 q(x))^{1/2}$, $\widetilde{\sigma}_{xx}(x,y) = \phi'(0)(\sigma_b^2 + \sigma_w^2 \lambda(x,y))^{1/2}$.
Then the processes $m_t^{(i)}$, $q_t^{(i)}$, $\lambda_t^{(i,j)}$ follow the SDEs:
\begin{align*}
    &dm_t^{(i)} = \widetilde{\mu}_x(x_t^{(i)}) dt + \widetilde{\sigma}_x(x_t^{(i)}) \frac{1}{D} \sum_{d=1}^{D} dB^{(i)}_{t,d}\\
    &dq_t^{(i)} = \left(2\widetilde{\mu}_x(x_t^{(i)})m_t^{(i)} + \widetilde{\sigma}_x^2(x_t^{(i)})\right)dt + 2\widetilde{\sigma}_x(x_t^{(i)}) \frac{1}{D} \sum_{d=1}^{D} x_{t,d}^{(i)} dB^{(i)}_{t,d}\\
    &d\lambda_t^{(i,j)} = \left(\widetilde{\mu}_x(x_t^{(i)})m_t^{(j)} + \widetilde{\mu}_x(x_t^{(j)})m_t^{(i)} + \widetilde{\sigma}_{xx}^2(x_t^{(i)},x_t^{(j)})\right)dt\\
    &\quad+ \frac{1}{D}\widetilde{\sigma}_x(x_t^{(i)})\sum_{d=1}^{D}x_{t,d}^{(j)}dB^{(i)}_{t,d} + \frac{1}{D}\widetilde{\sigma}_x(x_t^{(j)})\sum_{d=1}^{D}x_{t,d}^{(i)}dB^{(j)}_{t,d}
\end{align*}
where $m_t^{(i)} = \frac{1}{D}\sum_{d=1}^D x_{t,d}^{(i)}$.
\end{lemma}
\begin{proof}
This is again a direct consequence of multi-dimensional Ito's formula \citep[Section 4.2]{oksendal2003stochastic}.
\end{proof}

\begin{proof}[Heuristic for \cref{thm:layer_function_ode}]
From \cref{thm:sde_layers} we have
\begin{align*}
    &[m^{(i)}]_T = \frac{1}{D}\int_0^T \widetilde{\sigma}^2_x(x_t^{(i)}) dt\\
    &[q^{(i)}]_T = \frac{1}{D}\int_0^T 4 \widetilde{\sigma}^2_x(x_t^{(i)}) q_t^{(i)} dt\\
    &[\lambda^{(i,j)}]_T = \frac{1}{D}\int_0^T \widetilde{\sigma}^2_x(x_t^{(i)})q_t^{(j)} + \widetilde{\sigma}^2_x(x_t^{(j)})q_t^{(i)} dt
\end{align*}
where $\widetilde{\sigma}^2_x(x) = \phi'(0)^2(\sigma_b^2 + \sigma_w^2 q(x))$.
Assuming that $q_t^{(i)}$ can be controlled (for instance bounds on SDE solutions can be used to bound $\E[\sup_{0 \leq t \leq T} q_t^{(i)}]$ when $\phi''(0) = 0$) all the quadratic variations can be shown to converge to $0$ leaving out only the deterministic component.
The rest of \cref{thm:layer_function_ode} follows by assuming that the small noise limit of the SDEs is given by the corresponding ODEs, and by computing the ODEs solutions.
\end{proof}

\begin{proof}[Heuristic for \cref{thm:doubly_infinite_sde}]
We know from \cref{thm:layer_function_ode} that
\begin{align*}
    &\phi'(0)(\sigma_b^2 + \sigma_w^2 q_t^{(i)}) \rightarrow \phi'(0)(\sigma_b^2 + \sigma_w^2 q_t^{(i),\infty})\\
    &\frac{1}{2}\phi''(0)(\sigma_b^2 + \sigma_w^2 q_t^{(i)}) \rightarrow \frac{1}{2}\phi''(0)(\sigma_b^2 + \sigma_w^2 q_t^{(i),\infty})\\
    &\frac{\sigma_b^2 + \sigma_w^2 \lambda_t^{(i,j)}}{\big((\sigma_b^2 + \sigma_w^2 q_t^{(i)})(\sigma_b^2 + \sigma_w^2 q_t^{(j)})\big)^{1/2}} \rightarrow \frac{\sigma_b^2 + \sigma_w^2 \lambda_t^{(i,j),\infty}}{\big((\sigma_b^2 + \sigma_w^2 q_t^{(i),\infty})(\sigma_b^2 + \sigma_w^2 q_t^{(j),\infty})\big)^{1/2}}
\end{align*}
as $D \uparrow \infty$.
Then \cref{thm:doubly_infinite_sde} follows by assuming that the solution of \cref{eq:doubly_infinite_fc} converges to the solution of \cref{eq:doubly_infinite_fc_ode} as $D \uparrow \infty$ and by computing the transition density of \cref{eq:doubly_infinite_fc_ode}.
\end{proof}

\section{Additional plots}\label{app:additional_plots}

In \cref{fig:f_2d} we plot 2D function samples of $\xdt_{T,1}$ for $\mathcal{F}_{\tanh}$ and ${F}_{\swish}$ to complement the visualizations of Section 4.2.

\begin{figure}[!htb]
    \centering
    \includegraphics[width=0.45\linewidth]{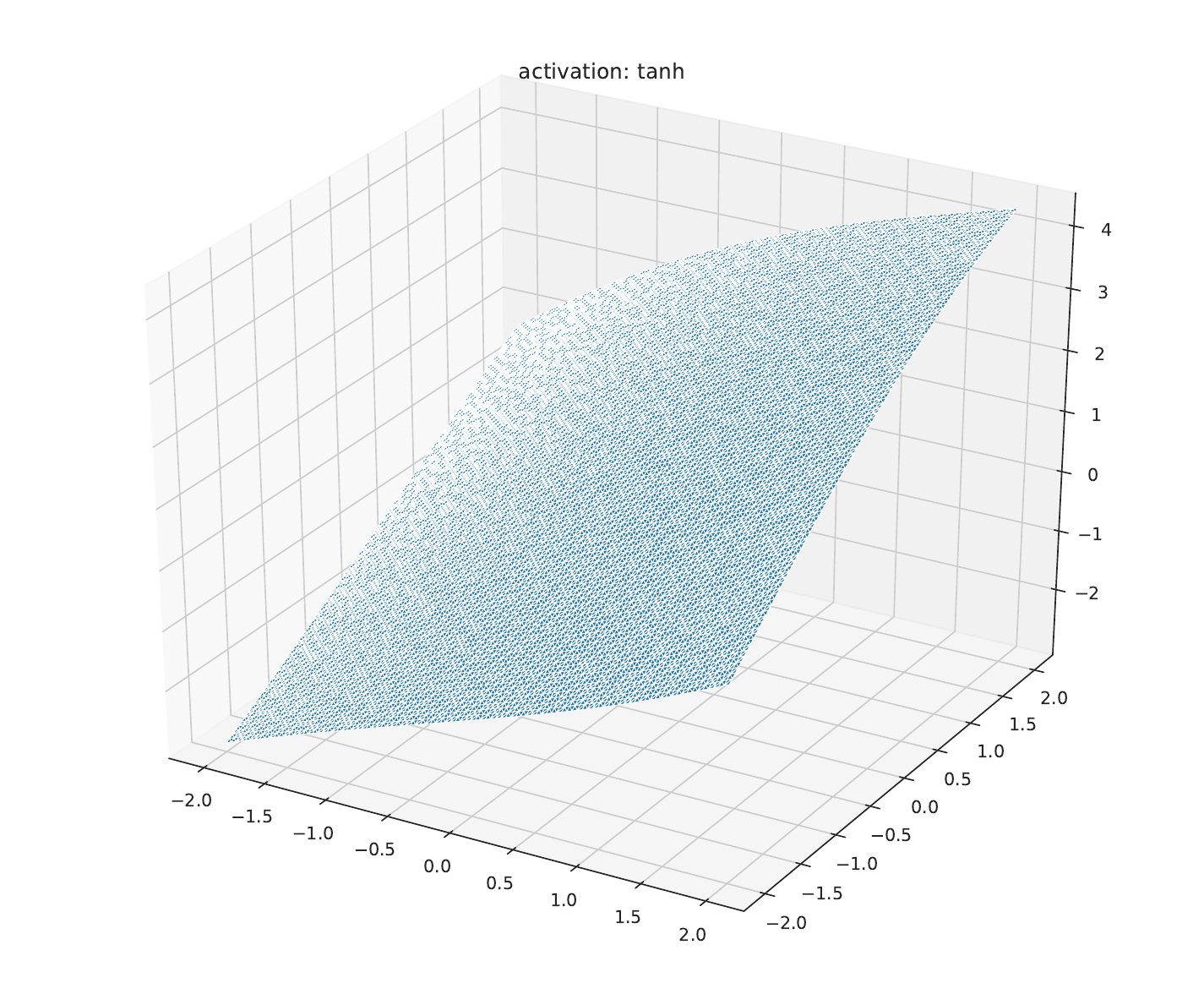}
    \includegraphics[width=0.45\linewidth]{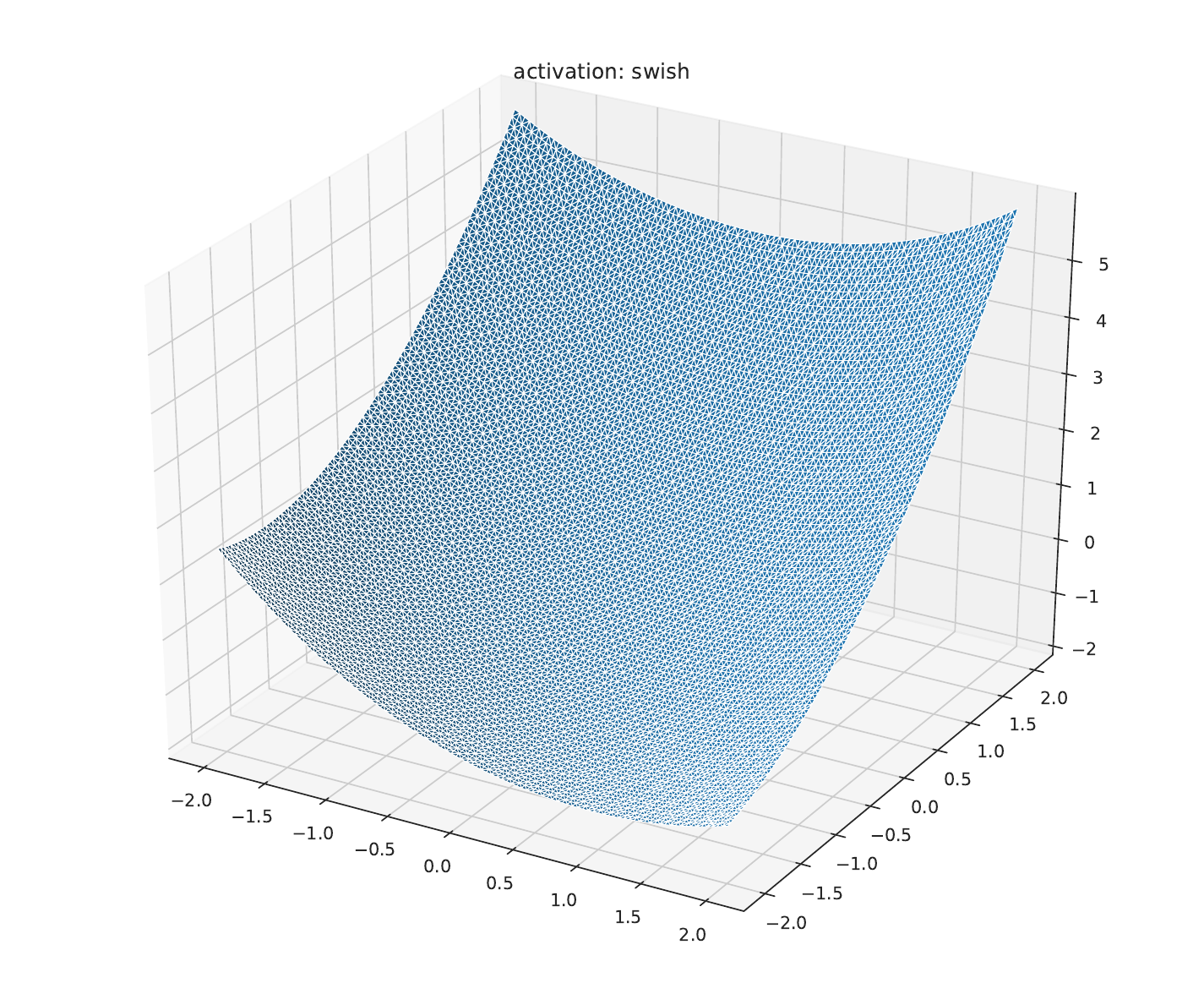}
    \caption{Function samples of $\xdt_{T,1}$ for $\mathcal{F}_{\tanh}$ (left) and ${F}_{\swish}$ (right) for $L=100$ and $D=100$ on the rectangle $[-2,2] \times [-2,2]$.}
    \label{fig:f_2d}
\end{figure}

\end{appendices}

\section*{Acknowledgements}

The authors are grateful to an Associate Editor and three anonymous Referees for all their constructive comments, corrections, and suggestions which improved remarkably the paper. Stefano Favaro is also affiliated to IMATI-CNR ``Enrico Magenes" (Milan, Italy). Stefano Favaro received funding from the European Research Council (ERC) under the European Union's Horizon 2020 research and innovation programme under grant agreement No 817257. Stefano Favaro gratefully acknowledge the financial support from the Italian Ministry of Education, University and Research (MIUR), ``Dipartimenti di Eccellenza'' grant 2018-2022.

\bibliography{refs}

\end{document}